\documentclass{article}

\usepackage[final]{neurips_2025}

\usepackage[utf8]{inputenc} 
\usepackage[T1]{fontenc}    
\usepackage{hyperref}       
\usepackage{url}            
\usepackage{booktabs}       
\usepackage{amsfonts}       
\usepackage{nicefrac}       
\usepackage{microtype}      
\usepackage{xcolor}         

\usepackage{hyperref}
\usepackage{graphicx}
\usepackage{float}
\usepackage{url}
\usepackage{amsthm}
\usepackage{amsmath}
\usepackage{caption}
\usepackage{subcaption}
\usepackage{comment}
\usepackage{booktabs}
\usepackage{siunitx}
\usepackage{multirow}

\usepackage{algorithm}
\usepackage{algpseudocode}

\newtheorem{theorem}{Theorem}[section]

\newtheorem{definition}{Definition}[section]

\def\rvx{{\mathbf{x}}}
\def\rvy{{\mathbf{y}}}

\newcommand{\inner}[2]{ \langle #1 | #2 \rangle }
\newcommand{\Inner}[2]{ \left\langle #1 \bigg| #2 \right\rangle }
\newcommand{\vect}[1]{\boldsymbol{#1}}
\newcommand{\mean}[2]{\mathbb{E}_{#2}\left[ #1 \right]}

\newcommand{\norm}[1]{\left\lVert#1\right\rVert}

\title{Entropic Time Schedulers for Generative Diffusion Models}

\author{%
  Dejan Stančević \thanks{dejan.stancevic@donders.ru.nl} \\
  Radboud University \\
    \And
   Florian Handke \\
   Ghent University - imec \\
   \AND
   Luca Ambrogioni \\
  Radboud University \\
}

\begin{document}
\maketitle

\begin{abstract}
The practical performance of generative diffusion models depends on the appropriate choice of the noise scheduling function, which can also be equivalently expressed as a time reparameterization. In this paper, we present a time scheduler that selects sampling points based on entropy rather than uniform time spacing, ensuring that each point contributes an equal amount of information to the final generation. We prove that this time reparameterization does not depend on the initial choice of time. Furthermore, we provide a tractable exact formula to estimate this \emph{entropic time} for a trained model using the training loss without substantial overhead. Alongside the entropic time, inspired by the optimality results, we introduce a rescaled entropic time. In our experiments with mixtures of Gaussian distributions and ImageNet, we show that using the (rescaled) entropic times greatly improves the inference performance of trained models. In particular, we found that the image quality in pretrained EDM2 models, as evaluated by FID and FD-DINO scores, can be substantially increased by the rescaled entropic time reparameterization without increasing the number of function evaluations, with greater improvements in the few NFEs regime. Code is available at
\url{https://github.com/DejanStancevic/Entropic-Time-Schedulers-for-Generative-Diffusion-Models}.

\end{abstract}

\section{Introduction} 
Generative diffusion models \citep{sohldickstein2015deepunsupervisedlearningusing}, and especially score-based diffusion models, have achieved state-of-the-art performance in image \citep{dhariwal2021diffusionmodelsbeatgans, rombach2022highresolutionimagesynthesislatent, song2021scorebasedgenerativemodelingstochastic} and video generation \citep{ho2022videodiffusionmodels, singer2022makeavideotexttovideogenerationtextvideo}. Generative diffusion models are obtained by reverting a forward diffusion process, which injects noise into the distribution of the data until all information has been lost. In practice, the performance of these models is highly dependent on the choice of a noise scheduling function that regulates the rate of noise-injection \citep{song2022denoisingdiffusionimplicitmodels}. In most commonly used models, a change of noise scheduling is mathematically equivalent to a change of time parameterization. From a theoretical perspective, the choice of time parametrization, or equivalently of noise scheduling, is not constrained by theory since any change of time in the forward process is automatically corrected in reverse dynamics \citep{song2021scorebasedgenerativemodelingstochastic}. However, as explained above, the choice of time is very important practically since it affects both the temporal weighting during training and the discretization scheme during inference. Consequently, an 'incorrect' choice of time variable can lead to severe inefficiencies due to the under-sampling of some temporal windows and the redundant over-sampling of others. This is particularly problematic since recent theoretical and experimental work suggested that 'generative decisions' tend to be clustered in critical time windows \citep{li2024criticalwindowsnonasymptotictheory}, which have been connected to symmetry-breaking phase transitions in physics \citep{raya2023spontaneoussymmetrybreakinggenerative, ambrogioni2024statisticalthermodynamicsgenerativediffusion, Biroli_2024, sclocchi2024phasetransitiondiffusionmodels}. The "triviality" of the first phase of diffusion prior to the initial phase transitions has led to the idea that this early phase can be skipped in one 'jump' using a pre-trained initialization \cite{lyu2022accelerating}. These late initialization schemes can be seen as a special case of time re-scheduling that compresses the high-noise part of the original schedule.

The idea of changing the diffusion time in a data-dependent way, also known as time-warping, was first introduced in \cite{dieleman2022continuous} in the context of a class of diffusion models for sequences of discrete tokens. However, their implementation required the used of special architectures trained with cross-entropy loss instead of the standard denoising score matching. In this paper, we show that a natural data-dependent time parametrization can be tractably obtained for any continuous generative diffusion model as the (rescaled) conditional entropy of $\rvx_0$ given $\rvx_t$. This choice of time leads to a constant entropy rate, meaning that each time point contributes to the final generation in an equal amount. Furthermore, we show that this \emph{entropic time} is invariant, meaning that it does not depend on the original choice of time parameterization. Examples of the same SDE in the entropic time and standard time are given in figure \ref{fig: Evolution of Distributions, 1-D}. Furthermore, inspired by the optimality results, we introduce a \emph{rescaled entropic time}. We provide an exact tractable formula that relates these quantities to the empirical EDM \citep{karras2022elucidating} and DDPM  \citep{song2022denoisingdiffusionimplicitmodels} loss, which can be used to easily define the entropic time for any given trained network.

\begin{figure}[h]
    \centering 
    \includegraphics[scale=0.45]{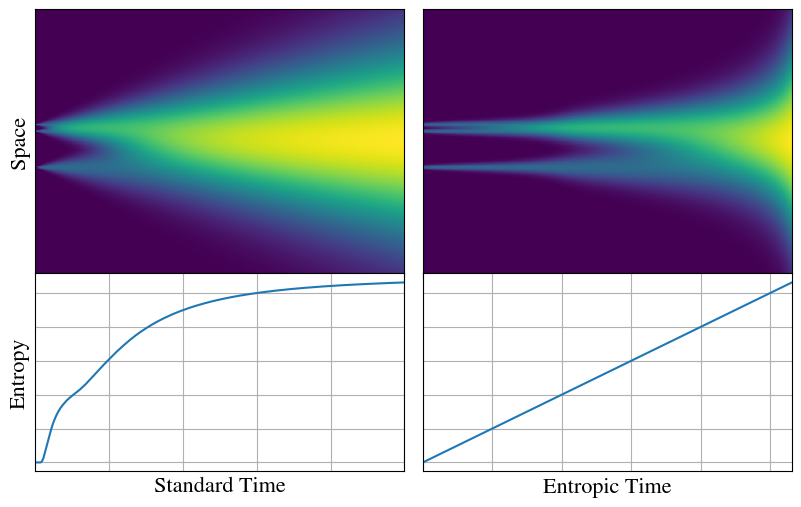}
    \caption{An example of the same SDE and its conditional entropy in the standard and entropic time.}
    \label{fig: Evolution of Distributions, 1-D}
\end{figure}

\section{Related work}
\textbf{Accelerated Sampling Procedures} One of the most significant challenges in current diffusion models is the slow generative process. Since the introduction of the connection between the diffusion models and SDEs \citep{song2021scorebasedgenerativemodelingstochastic}, a wide array of research has aimed to address this issue by designing better numerical integrators. Some of the research in that direction includes the works of \citet{liu2022pseudo} and \citet{lu2022fast}. An alternative line of research focuses on optimizing the sampling time schedule itself. \citet{sabour2024align} presents a principled approach to optimizing sampling schedules in diffusion models by aligning them with stochastic solvers, enabling higher efficiency. \citet{wang2024closer} splits the generation process into three categories (acceleration, deceleration, and convergence steps), identifies imbalances in time step allocation, and introduces methods to address them, leading to faster training and sampling. \citet{lee2024beta} uses spectral analysis of images to design a sampling strategy that prioritizes critical time steps, improving quality while reducing the number of steps. \citet{li2023autodiffusion} explores joint optimization of time steps and architectures for more efficient generation without additional training. While much of this research focuses on learning or empirically determining optimal sampling schedules, our work provides a more theoretical perspective based on ideas from information theory. The closest work to ours is \citet{dieleman2022continuous}, in which they use a cross-entropy loss to deduce a time-warping function for diffusion language models. However, our work differs since we analyze the standard diffusion models, where cross-entropy is not available. Furthermore, their expression for the entropy is not exact, as it implies an assumption of conditional independence of the tokens given the noisy state. On the other hand, here we provide exact formulas that can be applied to any generative diffusion model trained with denoising score matching, both in the continuous and in the discrete regime. 

\textbf{Connection Between Entropy, Information Theory, and Diffusion Models}
Diffusion models are inherently tied to concepts from information theory, particularly in the context of denoising Gaussian noise, which is a fundamental operation in information-theoretic frameworks. This connection has inspired a growing body of work exploring the interplay between diffusion models and information theory. \citet{premkumar2024neuralentropy} investigates entropy-based objectives for learning more robust generative models. \citet{kong2023information}, \citet{kong2023interpretable},  and \citet{franzese2025latent} aim to provide a clearer understanding of diffusion models through an information-theoretic lens. Although these works explore the connection between information theory and diffusion models, and employ similar equations to ours, our focus diverges slightly. We use information theory as a guide to design better sampling algorithms. Work exploring a similar direction to ours is \citet{li2025improving}. However, they explore the conditional entropy between two consecutive time steps given a fixed discretization grid, while we look at the conditional entropy between the current time step and time zero in a way that is invariant under the change of time and discretization.

\section{Background on score-matching generative diffusion}
The mathematics of generative diffusion models can be elegantly formalized in term of stochastic differential equations (SDE). Consider a target distribution $p(\rvx_0)$ defined by a data source such as a distribution of, for example, natural images, sound waves, or linguistic strings. We interpret this data source as the initial distribution of a diffusion process governed by the SDE:
\begin{equation} \label{eq: forward sde}
    dX_t = \vect{f}(X_t, t) dt + g(t) dW~,
\end{equation}
where $dW$ is a standard Wiener process, $\vect{f}(X_t, t)$ is a vector-valued drift function, and $g(t)$ is a scalar volatility function, which regulates the standard deviation of the input noise. The marginal densities of the process can be obtained from the Fokker-Planck equation:
\begin{equation}
    \partial_t p_t(\rvx_t) = \sum _{j=1}^d  \partial_{x_j} \left( - f_j(\rvx_t, t)  + \frac{g^2(t)}{2} \partial_{x_j} \right) p_t(\rvx_t)~,
\end{equation}
where $\partial_t$ is the partial derivative with the respect to time and $\partial_{x_j}$ is the partial derivative with respect to the $j$-th component of $\rvx_t$. 
We denote the forward "solution kernel" of the diffusion process as $p(\rvx_t \mid \rvy)$, which is the solution of the Fokker-Plank equation for $p_0(\rvy) = \delta(\rvy - \rvx_0)$. The core idea of generative diffusion is to sample from $\rvx_0$ by initializing an asymptotic noise state $\rvx_T$ (where $T$ is large enough for the SDE to reach its stationary distribution) and by "inverting" the temporal dynamic. This can be done using the reverse SDE:
\begin{equation} \label{eq: backward sde}
    dX_t = \left(f(X_t, t) - g(t)^2 \nabla \log{p(X_t)} \right)dt + g(t) d\tilde{W}~,
\end{equation}
which can be proven to give the same marginal densities of eq. \ref{eq: forward sde} when initialized with the appropriate stationary distribution, which is usually Gaussian white noise. We denote the reverse solution kernel of the reverse dynamics as $q(\rvx_0 \mid \rvx_t)$, which can be interpreted as the optimal denoising distribution. The data-dependent key component of the reverse dynamics is the so-called \emph{score function}, which can be written as an expectation over the optimal denoising distribution:
\begin{equation}
    \nabla \log{p_t(\rvx_t)} = \mean{\nabla \log{p_t(\rvx_t \mid \rvx_0)}}{q(\rvx_0 \mid \rvx_t)}~.
\end{equation}
In most practical forms of generative diffusion, the score function is approximated using a deep network $\vect{s}_{\vect{\theta}}(\rvx_t, t)$, where the parameters $\vect{\theta}$ are optimized by minimizing an upper bound on the quadratic score-matching loss:
\begin{equation}
\begin{aligned}
    &\mathcal{L}_{\text{SM}}(\vect{\theta}) \equiv \mean{\norm{\vect{s}_{\vect{\theta}}(\rvx_t, t) - \nabla \log{p_t(\rvx_t)} }^2 }{p_0(x_0), t \sim \lambda(t)} \\
    &\leq \mean{\norm{\vect{s}_{\vect{\theta}}(\rvx_t, t) - \nabla \log{p_t(\rvx_t \mid \rvx_0)} }^2 }{p_0(x_0), p(x_t \mid x_0), t \sim \lambda(t)} \equiv \mathcal{L}_{\text{DSM}}(\vect{\theta})
\end{aligned}
\label{eq: score matching}
\end{equation}
where $\lambda(t)$ is a density defined on the time axis. Note that $\mathcal{L}_{\text{SM}}(\vect{\theta})$ and $\mathcal{L}_{\text{DSM}}(\vect{\theta})$ differ only by a constant and therefore have the same gradients and optima. However, $\mathcal{L}(\vect{\theta})$ is substantially more tractable as it does not require samples from the unknown optimal denoiser $q(\rvx_0 \mid \rvx_t)$.

\section{Optimal sampling schedule as a change of time}

In this section, we revisit a result from \citealt{sabour2024align} and notice some interesting features. Inspired by it, we formalize what we mean by the change of time.

Obtaining the analytical expression for the optimal sampling schedule is difficult and, in most practical cases, impossible. However, \citealt{sabour2024align} shows that for the EDM noise schedule \citep{karras2022elucidating}, the optimal sampling schedule for the ODE flow when data comes from a normal distribution with variance $c^2$ has an analytical expression. More precisely, the sampling schedule, $[t_{min}, t_1, ... t_{max}]$, that minimize the KL divergence is given by
\begin{equation}
\label{eq: optimal schedule}
     \arctan\left( \frac{t_i}{c} \right) = \alpha_{min} + \frac{i}{N} (\alpha_{max} - \alpha_{min})
\end{equation}
where $\alpha_{min/max} = \arctan \left( \frac{t_{min/max}}{c} \right)$ (see theorem 3.1 in \citealt{sabour2024align}). It turns out that this schedule is also optimal for the deterministic DDIM \citep{song2022denoisingdiffusionimplicitmodels}, which we show in Appendix \ref{appendix: DDDIM schedule}. Moreover, since the DDIM solver is invariant under time change \citep{lu2022fast}, its optimal schedule remains invariant under time change, making it particularly suitable for comparing different time parameterizations.

This implies that even in the simple case, the optimal schedule depends on the data distribution. In addition, this result frames the optimization of a sampling schedule as a problem of time change. Rather than selecting timesteps differently for different numbers of sampling steps (e.g. EDM scheduler), theorem 3.1 shows that one should think of the sampling schedule as a transformation of time such that the sampling schedule becomes linear in the new time. Furthermore, in section \ref{section: isoentropic time}, we will connect equation \ref{eq: optimal schedule} with the conditional entropy production.

\subsection{Change of time}
The change of time in SDEs is a powerful technique used to simplify their analysis and solutions. By altering the time variable, the dynamics of the SDE can be transformed into a more manageable form. More information can be found in section $8.3$ in \citet{lawler2010stochastic}.
\begin{definition}
    We say a function $\phi$ is a proper time change if it is continuous and strictly increasing.
\end{definition}
It can be shown that given a proper time change, $f$, and a random process, $X_t$, that solves the SDE
$
    dX_t = f_t(X_t, t) dt + g_t(t) dW_t~,
    \label{eq: SDE1}
$
then $Y_t=X_{\phi(t)}$ solves
$
    dY_t = \dot{\phi}(t) f_t(Y_t, \phi(t)) dt + \sqrt{\dot{\phi}(t)} g_t(\phi(t)) dW_{t}~.
$
Guided by the theory of time change, we define an equivalence between SDEs.

\begin{definition} \label{definition: equivalence up to a time change}
Given two SDEs
$
    dX_t = f(X_t, t) dt + g(t) dW_t
$
\text{and} $
    dX_s = \tilde{f}(X_s, s) ds + \Tilde{g}(s) dW~,
    \label{eq: SDE2}
$
we say that they are equivalent up to a time change if there exists a proper time change, $\phi: t \mapsto s$, such that
\begin{enumerate}
    \item $ \dot{\phi}(t) \Tilde{f}(x, \phi(t)) = f(x, t) $
    \item $ \sqrt{\dot{\phi}(t)} \Tilde{g}(\phi(t)) = g(t) $.
\end{enumerate}
\end{definition}
Furthermore, we can require $f(0) = 0$ without affecting anything (since it is equivalent to subtracting a constant from the original function). By requiring that, we get that a time change between two SDEs is unique if it exists. Under a time change, the forward kernels stay the same, in the sense that $p_t(x|x_0) = q_{\phi(t)}(x|x_0)$ holds (this follows from $Y_t=X_{\phi(t)}$). Essentially, time change squeezes and stretches the time axis but does not fundamentally change the diffusion process. Algorithm~\ref{algo: Time Change Sampling} shows how to implement sampling using time change.

Given this notion of equivalence, a natural question arises: Is there a preferred or canonical time parameterization? We argue that a conditional entropy, $\mathbf{H}[x_0|x_t]$, and quantities derived from it are good candidates. However, for $\mathbf{H}[x_0|x_t]$ to make sense, we assume that we are given an initial distribution, $p_0(x)$ (that is, a data set). Therefore, besides an SDE, we require a dataset for the entropic time. In the further text, we will always assume that the dataset is given and is the same for different time parameterizations of SDEs.

\begin{algorithm}
\caption{Sampling using time change}
\label{algo: Time Change Sampling}
\begin{algorithmic}[1] 
\Procedure{TimeChangeSampler}{$\{t_i\}_{i=0}^M, \{\phi(t_i)\}_{i=0}^M, \sigma(t), s(t), D_\theta(x,\sigma), \newline \text{solver}(x, D_\theta, \sigma_{\text{cur}},\sigma_{\text{next}}, s_{\text{cur}}, s_{\text{next}}), N$}
\State $\tau_j \leftarrow \phi(t_0) + \frac{j}{N-1} (\phi(t_M)-\phi(t_0)) \text{ for } j=0,\dots N-1$ \Comment{Uniform spacing in new time}
\State $\tilde{t}_j \leftarrow \text{interp}(\tau_j; \{\phi(t_i)\}_i^M, \{t_i\}_i^M) \text{ for } j=0,\dots N-1$
\Comment{Corresponding old time}
\State $\tilde{\sigma}_j, \tilde{s}_j \leftarrow \sigma(\tilde{t}_j), s(\tilde{t}_j) \text{ for } j=0,\dots N-1$
\State $\tilde{\sigma}, \tilde{s} \leftarrow \{ 0, \tilde{\sigma} \}, \{ 1, \tilde{s} \}$
\State \textbf{sample} $x \sim \mathcal{N}(0, \tilde{\sigma}_{N}^2I)$
\For{j $\in \{N,\dots,1\}$}
\State $\sigma_{\text{cur}}, \sigma_{\text{next}} \leftarrow \tilde{\sigma}_j, \tilde{\sigma}_{j-1}$
\State $s_{\text{cur}}, s_{\text{next}} \leftarrow \tilde{s}_j, \tilde{s}_{j-1}$
\State $x \leftarrow \text{solver}(x, D_\theta, \sigma_{\text{cur}}, \sigma_{\text{next}}, s_{\text{cur}}, s_{\text{next}})$ \Comment{e.g. Heun, DDIM, etc.}
\EndFor
\State \textbf{output} $x$
\EndProcedure
\end{algorithmic}
\end{algorithm}

\section{Entropic time schedules}

In this section, we introduce the concepts of entropic time and rescaled entropic time. First, we provide some reasons for using the conditional entropy as a new time parameterization. Then, we show how to obtain the conditional entropy in practice and show its connection with commonly used quantities in diffusion literature. Furthermore, we demonstrate that the entropic time parameterizations are well-defined and invariant under the initial time parameterization of the SDE. There are several possible choices for the entropy function, which highlight different aspects of information transfer. The most straightforward choice is the information transfer $T_t$. Consider an initial source $\rvx_0 \sim p_0$ is transmitted through a noisy channel $p(\rvx_t \mid \rvx_0)$, which is determined by the solution of the SDE given in eq. \ref{eq: forward sde}. The noise-corrupted signal is received and decoded using $q(\rvx_0 \mid \rvx_t)$. The amount of information transferred at time $t$ can be quantified as the difference between the prior and posterior entropy:
\begin{equation} \label{eq: information transfer}
    \mathbf{T}_t = \mathbf{H}[\rvx_0] - \mathbf{H}[\rvx_0 | \rvx_t] = \mathbf{I}[\rvx_0; \rvx_t]
\end{equation}
 where $\mathbf{H}[\rvx_0] = \mean{\log{p(\rvx_0)}}{p_0(x_0)}$ is the entropy of the source, $\mathbf{H}[\rvx_0|\rvx_t] = \mean{\log{p(\rvx_0|\rvx_t)}}{p(x_0, x_t)}$ is the conditional entropy under the optimal denoising distribution, and $\mathbf{I}[\rvx_0; \rvx_t]$ is a mutual information. Therefore, it is natural to interpret this quantity as the amount of information available at time $t$ concerning the identity of the source data. Up to a constant shift, this is equivalent to using the time variable $\phi(t) = \mathbf{H}[\rvx_0|\rvx_t]$ in the forward process. This time axis is defined by having a constant conditional entropy rate between the final generated image and the noisy state at time $t$.

\subsection{Characterizing the conditional entropy} \label{section: Estimating a conditional entropy}

Having established that a conditional entropy makes sense as a new time parameterization, a question arises: How do we calculate it in practice? In general, conditional entropy can be written as
$
    \mathbf{H}[\rvx_0|\rvx_t] = \mathbf{H}[\rvx_0] - \mathbf{I}[\rvx_0; \rvx_t] = \mathbf{H}[\rvx_0] + \mathbf{H}[\rvx_t|\rvx_0] - \mathbf{H}[\rvx_t]~.
$

In practice, $\mathbf{H}[\rvx_t|\rvx_0]$ is easy to get once the forward kernel is known, but it is difficult to obtain a numerical value of $\mathbf{H}[\rvx_t]$. However, by looking at a time derivative of the conditional entropy, we get a method for obtaining a numerical value. The time derivative is given by
\begin{equation}
\begin{aligned}
    \dot{\mathbf{H}}[\rvx_0|\rvx_t] = \dot{\mathbf{H}}[\rvx_t|\rvx_0] - \dot{\mathbf{H}}[\rvx_t].
\end{aligned}
\label{eq: conditional entropy production}
\end{equation}
Hence, to know the time derivative, we need to calculate the time derivative of $\mathbf{H}[\rvx_t]$. In case when an SDE is given by \ref{eq: SDE1}, the entropy production is given by
\begin{equation}
    \dot{\mathbf{H}}[\rvx_t] = \mathbb{E}_{p_t(x_t)} [\nabla (f_t)] + \frac{g^2_t}{2} \mathbb{E}_{p_t(x_t)} [|| \nabla \log p(\rvx_t) ||^2 ].
    \label{eq: entropy production}
\end{equation}
The equation is a well-known expression in nonequilibrium thermodynamics for entropy production \citep{premkumar2024neuralentropy}. The derivation of the expression can be found in the appendix \ref{section: entropy production}. Similarly, we can obtain the similar expression for $\mathbf{H}[\rvx_t|\rvx_0]$. Combining these two expressions, we obtain
\begin{equation}
\begin{aligned}
    \dot{\mathbf{H}}[\rvx_0|\rvx_t] = \frac{g^2_t}{2} \left( \mathbb{E}_{p(x_t,x_0)} [|| \nabla \log p(\rvx_t|\rvx_0) ||^2 ] - \mathbb{E}_{p_t(x_t)} [|| \nabla \log p(\rvx_t) ||^2 ] \right).
\end{aligned}
\label{eq: conditional entropy in terms of norms}
\end{equation}
Note that this expression depends on the data distribution only through the Euclidean norm of the score function, which is approximated by a neural network in diffusion models.

\subsection{Estimating the entropy rate from the training loss}
In this section, we present a connection between the conditional entropy rate and training loss. For more details on the derivation of these results, see the Appendix \ref{section: connection with a quadratic error and loss}. In practice, most diffusion models can be written using the framework introduced in \citet{karras2022elucidating}. In this framework, the SDE is written as
$
    dX_t = \frac{\dot{s}(t)}{s(t)} X_t dt + s(t) \sqrt{2 \dot{\sigma}(t) \sigma(t)} dW~,
$
with $p(x_t|x_0) = \mathcal{N}(x_t; s(t) x_0, s(t)^2 \sigma(t)^2 I)$ as a forward kernel. This leads to the following conditional entropy production 
\begin{equation}
\begin{aligned}
\label{eq: EDM conditional entropy}
    \dot{\mathbf{H}}[\rvx_0|\rvx_t] & = \frac{D \dot{\sigma}(t)}{\sigma(t)} - s(t)^2 \dot{\sigma}(t) \sigma(t) \mathbb{E}_{p_t(x_t)} [|| \nabla \log p_t(\rvx_t) ||^2 ]
\end{aligned}
\end{equation}
where $D$ is a dimension of the space (e.g. for the MNIST dataset, it would be $28^2$). In the rest of this paper, we will be using this framework. The squared error, $\epsilon_t^2$, encapsulates our uncertainty at time $t$ about the final sample $x_0$ and is given by
\begin{equation}
\begin{aligned}
    \epsilon_t^2 & = \mathbb{E}_{p_t(x_t)} [ \mathbb{E}_{p(x_0 | x_t)}[||\rvx_0 - \mathbb{E}_{p(y_0 | x_t)}[\rvy_0]||^2] ] 
     = \mathbb{E}_{p_t(x_t)} [tr(\sigma_{\rvx_0|\rvx_t}^2)].
\end{aligned}
\end{equation}
Using the fact that we can write $\sigma_{\rvx_0|\rvx_t}^2$ as
$
    \sigma(t)^2 (I + s(t)^2 \sigma(t)^2 H[\log p_t(x_t)])
$ (see Appendix \ref{section: Tweedie second order}),
we get
\begin{equation}
\begin{aligned}
    \dot{\mathbf{H}}[\rvx_0|\rvx_t] = \frac{\dot{\sigma}(t)}{\sigma(t)^3} \epsilon_t^2.
\end{aligned}
\label{eq: connection between quad_error and conditional entropy production}
\end{equation}
Recognizing that $\dot{SNR} = - \frac{\dot{\sigma}}{\sigma^3}$ and $\dot{\mathbf{H}}[\rvx_0|\rvx_t] = - \dot{\mathbf{I}}[\rvx_t;\rvx_0]$, equation \ref{eq: connection between quad_error and conditional entropy production} is precisely the well-known $I$--MMSE identity from information theory \citep{guo2005mutual}.

Furthermore, integrating over time yields an expression for the conditional entropy:
\begin{equation}
\begin{aligned}
    \mathbf{H}[\rvx_0|\rvx_1] 
    = - \int_0^1 \dot{\mathrm{SNR}}(t)\, \epsilon_t^2 \, dt.
\end{aligned}
\end{equation}
This expression coincides with the continuous-time (infinite-step) limit of the variational lower bound derived by \citet{kingma2021variational}, revealing a direct information-theoretic characterization as already hinted by \citet{kong2023information}.

The previous results provide a simple way of estimating the conditional entropy rate from the standard loss function of a trained diffusion model due to a close connection between the squared error and the loss. This provides a tractable way to estimate the conditional entropy from the training error. Note that, using the error of the model entails an approximation since the entropy is defined with respect to the true score function and, therefore, does not take into account the discrepancy between the learned and true score.

To analyze this deviation, we start from a striking result: the conditional entropy production is, up to a multiplicative factor, the gap between the explicit and denoising score matching loss in \ref{eq: score matching}! In fact, following the steps from \citet{vincent2011connection} and keeping track of the terms that are constant in $\boldsymbol{\theta}$, we have
\begin{equation}
\begin{aligned}
    \mathcal{L}_{\text{SM}}(\vect{\theta}) & = \mathcal{L}_{\text{DSM}}(\vect{\theta})  - \mathbb{E}_{p(x_0, x_t), \lambda(t)}[\norm{\nabla \log p(\rvx_t|\rvx_0)}^2 - \norm{\nabla \log p(\rvx_t)}^2].
\end{aligned}
\end{equation}
Using expression \ref{eq: conditional entropy in terms of norms}, we can rewrite the above equality as
\begin{equation}
    \mathcal{L}_{\text{DSM}}(\vect{\theta}) = \mathcal{L}_{\text{SM}}(\vect{\theta}) + \mathbb{E}_{\lambda(t)} \left[ \frac{2}{g_t^2} \dot{\mathbf{H}}[\rvx_0|\rvx_t] \right].
\end{equation}
This relation can also be expressed at a single time point $t$ as
\begin{align}
    \dot{\mathbf{H}}[\rvx_0|\rvx_t] + \frac{g_t^2}{2}\,\delta_t^2(\boldsymbol{\theta}) 
    = \frac{g_t^2}{2}\, 
    \mathbb{E}_{\rvx_t, \rvx_0}\!\left[
        \big\|\vect{s}_{\boldsymbol{\theta}}(\rvx_t, t) 
        - \nabla \log p_t(\rvx_t \mid \rvx_0)\big\|^2
    \right],
\end{align}
where $ \delta_t^2(\boldsymbol{\theta}) 
    = \mathbb{E}_{p_t(\rvx_t)} 
    \big[\,
        \|\vect{s}_{\boldsymbol{\theta}}(\rvx_t, t) 
        - \nabla \log p_t(\rvx_t)\|^2
    \big] $
denotes the mean squared error between the true score and its neural approximation.  
The right-hand side corresponds to our estimate of the conditional entropy production.  
It follows that the estimated entropy production always upper-bounds the true value, with the gap determined by the disagreement between the learned and true scores, $\delta_t^2(\boldsymbol{\theta})$.  
In this sense, $\dot{\mathbf{H}}[\rvx_0|\rvx_t]$ can be interpreted as the \emph{irreducible} contribution to the loss, reflecting the intrinsic uncertainty of the optimal denoising process.

\subsection{The entropic and rescaled entropic times}
\label{section: isoentropic  time}

Here, we introduce a rescaled entropy and show that both rescaled entropy and conditional entropy are proper changes of time and are invariant under different time parameterizations of SDE. Proofs can be found in the Appendix \ref{section: proofs}. First, we notice that in the case of continuous data, the conditional entropy goes to negative infinity at time equal to zero. In practice, this is not observed since diffusion models always start from a non-zero initial time. However, it adds arbitrariness to the overall curve of the conditional entropy. To combat this problem, guided by the observation that the change of time for the optimal sampling schedule for normally distributed data, eq. \ref{eq: optimal schedule}, is equal to the rescaled entropy (see Appendix \ref{appendix: rescaled entropy}), we introduce a rescaled entropy as $\int_0^t \sigma(\tau) \dot{\mathbf{H}}[\rvx_0|\rvx_{\tau}] d\tau$. Algorithm~\ref{algo: Rescaled Entropy} shows how to estimate rescaled entropy in practice (and how it was estimated in this work).
\begin{theorem}
    Given an SDE and initial data distribution $p_0(x)$, $\phi(t) = \mathbf{H}[x_0|x_t]$ and $\phi(t) = \int_0^t \sigma(\tau) \dot{\mathbf{H}}[\rvx_0|\rvx_{\tau}] d\tau$ are proper time changes.
\end{theorem}
We call these time parameterizations an \textit{entropic time} and \textit{rescaled entropic time}, respectively. Naturally, an important question emerges: How does the time parameterization of an initial SDE influence its reparameterized form? We show that an SDE written in entropic time is unique and does not rely on its initial parameterization. More precisely, given two SDEs equivalent up to a time change, the SDEs expressed in their respective entropic times are equivalent up to a time change, with the time change being the identity function (i.e. drift and noise terms of SDEs in entropic times are related by conditions 1. and 2. from definition \ref{definition: equivalence up to a time change}, and are the same since the time derivative of the time change is one).
\begin{theorem}
    Given two SDEs as given in definition \ref{definition: equivalence up to a time change}, and following time changes
    \begin{enumerate}
    \item $ \phi: t \mapsto s=f(t) $
    \item $ \Phi_t: t \mapsto \mathbf{H}_t[\rvx_0|\rvx_t] $
    \item $ \Phi_s: s \mapsto \mathbf{H}_s[\rvx_0|\rvx_s] $,
\end{enumerate}
    it follows that
    $$ F := \Phi_s \circ \phi \circ \Phi_t^{-1}: \mathbf{H}_t[\rvx_0|\rvx_t] \mapsto \mathbf{H}_s[\rvx_0|\rvx_s] $$
    is a proper time change implementing the equivalence and is equal to the identity map, $F = id$.
\end{theorem}
A similar result holds for the rescaled entropic time as well. Therefore, once reparameterized in entropic time (or rescaled entropic time), no matter the starting SDE time parameterization, drift and noise are always the same.

\begin{algorithm}
\caption{Estimation of rescaled entropy, $\int \sigma(\tau)\dot{\mathbf{H}}[\mathbf{x}_0|\mathbf{x}_\tau]d\tau$}
\label{algo: Rescaled Entropy}
\begin{algorithmic}[1] 
\Procedure{EstimateRescaledEntropy}{$D_\theta(x,\sigma)$, $\sigma(t)$, $s(t)$, $\{t_i\}_{i=0}^N$, $M$}
\State \textbf{sample} $x_0^{j} \sim p_0 \text{ for } j = 1, \dots, M$
\State $\mathbf{x}_0 \leftarrow [x_0^1,\dots, x_0^M]$ 
\State $\mathbf{R}_{i:N} \leftarrow \mathbf{0}$
\For{i $\in \{0, \dots, N-1\}$}
    \State \textbf{sample} $\nu_j \sim \mathcal{N}(\mathbf{0},I) \text{ for } j = 1, \dots, M$
    \State $\boldsymbol{\nu} \leftarrow [\nu_1,\dots,\nu_M]$
    \State $\mathbf{x}_{t_i} \leftarrow s(t_i)  \mathbf{x}_0 + s(t_i)\sigma(t_i) \boldsymbol{\nu}$ 
    \State $\hat{\mathbf{x}}_0 \leftarrow D_\theta(\mathbf{x}_{t_i}, \sigma(t_i))$
    \State $\epsilon^2 \leftarrow \frac{1}{M} \sum_{j=1}^M \norm{\hat{\mathbf{x}}^j_0-\mathbf{x}^j_0}^2$
    \State $\mathbf{R}_{i+1} \leftarrow \mathbf{R}_{i} + \frac{\dot{\sigma}(t_i)}{\sigma^2(t_i)} (t_{i+1}-t_i) \epsilon^2$ \Comment{Riemann sum}
\EndFor
\State \textbf{output} $\mathbf{R}_{i:N}$
\EndProcedure
\end{algorithmic}
\end{algorithm}

\subsection{Spectral rescaled entropic time}

We based our definition of rescaled entropic time on optimality results for isotropic Gaussian distributions. However, these results do not account for how different directions in an anisotropic Gaussian influence the optimal schedule. From equation \ref{eq: conditional entropy in terms of norms}, we observe that the total entropy production can be interpreted as a sum of contributions from all basis directions, where the basis can be any orthonormal set (since only norms of scores affect the production). This interpretation corresponds to one specific way of weighting different directions. For image data and diffusion in pixel space, we explore an alternative in this paper: setting the rescaled entropy in each Fourier basis direction to be equal to 1 at the final time (i.e., giving an equal importance for each frequency), and then weighting them by their respective amplitudes. Theorems from the previous section still hold for the spectral rescaled entropy since they hold for each frequency (basis). An example of the resulting rescaled entropy across different frequencies is shown in figure \ref{fig: normalized spectral RE}.

\begin{figure}[H]
    \centering
    \includegraphics[width=0.75\linewidth]{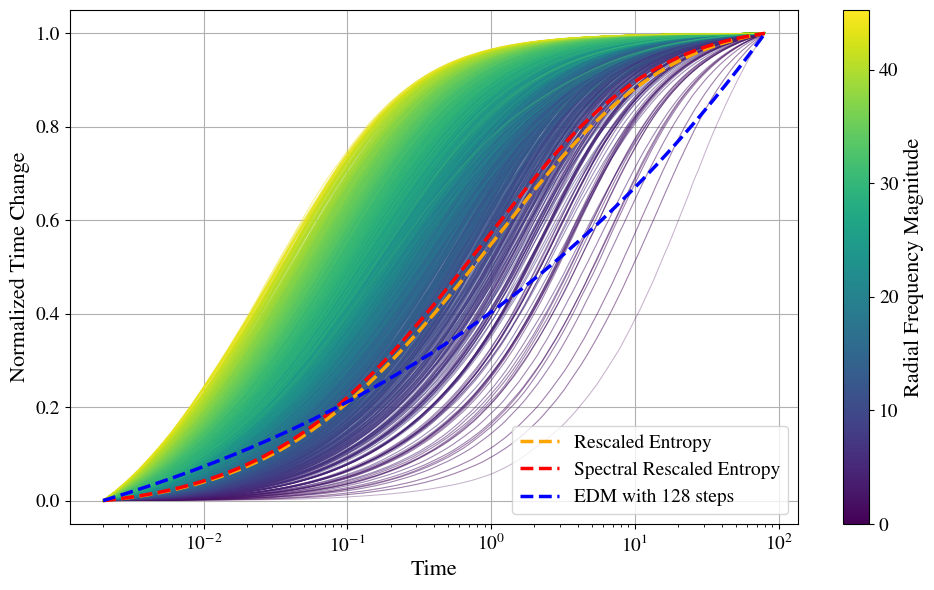}
    \caption{Normalized rescaled entropy as a function of radial frequency for the red channel in ImageNet-64, together with normalized rescaled entropy, spectral rescaled entropy, and EDM with 128 steps.}
    \label{fig: normalized spectral RE}
\end{figure}

\section{Experiments}
We compare the performance of a few-step generation in the standard, entropic, and rescaled entropic times for several low-dimensional examples where an analytic expression for a score is easy to calculate. Next, we compare the performance of trained EDM and EDM2 models \citep{karras2022elucidating, karras2024analyzing} on CIFAR10 \citep{krizhevsky2009learning}, FFHQ\citep{karras2019style}, and ImageNet \citep{russakovsky2015imagenet} using the FID \citep{heusel2017gans} and FD-DINOv2 \citep{oquab2023dinov2, stein2023exposing} scores. More details about the setup can be found in the appendix \ref{appendix: experiments}.

\subsection{One-dimensional experiments}

We used an analytic expression of a score function to compare the performance of a few-step generation process in different time parameterizations in one dimension. We used equidistant steps in the standard, entropic, and rescaled entropic times. We used the stochastic DDIM solver \citep{song2022denoisingdiffusionimplicitmodels}. We compared those schedules for discrete data and a mixture of Gaussians. We used the Kullback-Leibler divergence to compare results for different schedules. An example of KL divergence behavior against the number of generative steps is given in figure \ref{fig: example KL}. In general, we can see that in the discrete case, the entropic time outperforms other schedules by a large margin, while the standard schedule gives the worst results. Furthermore, we noticed that when variances of Gaussians are much smaller than the distance between them (i.e. there is no significant overlap between Gaussians), the entropic schedule gives better results. However, when the variances are not negligible in the mixture of Gaussians case, we can see that the rescaled entropic schedule gives the best results, while the entropic schedule underperforms. This suggests that the entropic time might significantly improve certain discrete diffusion models.

\begin{figure}[H]
    \centering
    \begin{minipage}{0.85\textwidth}
        \centering
        \begin{subfigure}[b]{0.45\textwidth}
            \centering
            \includegraphics[width=\textwidth]{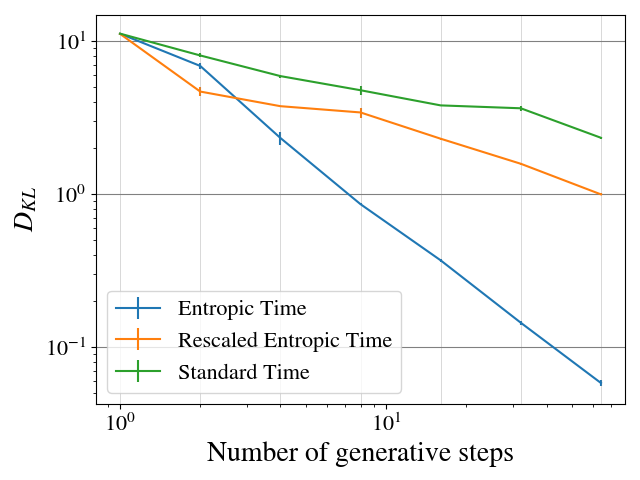}
            \caption{Discrete}
            \label{fig:top-a}
        \end{subfigure}
        \hfill
        \begin{subfigure}[b]{0.45\textwidth}
            \centering
            \includegraphics[width=\textwidth]{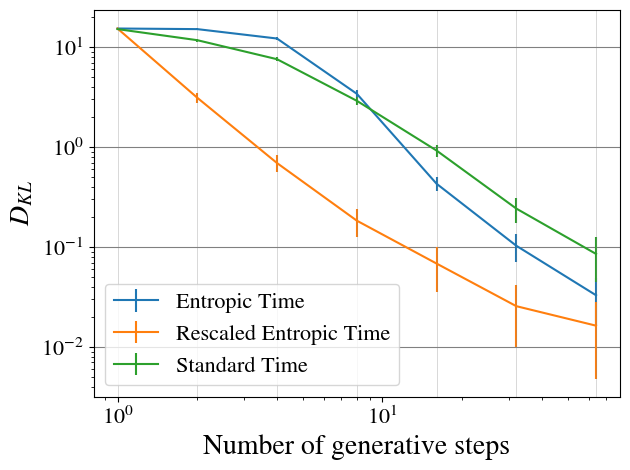}
            \caption{Continuous}
            \label{fig:top-b}
        \end{subfigure}
        \caption{Kullback–Leibler divergence against the number of generative steps for different time parameterizations for mixture of $15$ data points (discrete) and $15$ Gaussians (continuous).}
        \label{fig: example KL}
    \end{minipage}
\end{figure}


\subsection{CIFAR10, FFHQ, and ImageNet}

We compared the performance for different numbers of generative steps using standard, entropic, rescaled entropic, and spectral rescaled entropic (for diffusion in pixel space) times. To sample, we used the deterministic and stochastic DDIM solver. For CIFAR-10 and FFHQ, the EDM unconditional VP models were used \citep{karras2022elucidating}. For ImageNet-64, the EDM2-S and EDM2-L models were used, while for ImageNet-512, the EDM2-XS and EDM2-XXL models were used. For ImageNet-512, we used both models: one optimized for FID and the other for FD-DINOv2.

\noindent
\begin{minipage}[t]{0.4\textwidth}
Figure \ref{fig: example of generated images} gives an example of the effect of different schedules on generated images. We observed that the rescaled entropic schedule produces images with lower brightness. The results for FFHQ are presented in table \ref{table:FFHQ-FID}, while the FID scores for FID-optimized networks and the DINOv2 scores for DINOv2-optimized networks on ImageNet-512 are reported in table \ref{table: ImageNet-512 results}. We note that the using the rescaled entropy schedule, model EDM-XS beats the FD-DINOv2 result pro-
\end{minipage}
\hfill
\begin{minipage}[t]{0.58\textwidth}
\centering
\captionof{table}{FID scores for different sampling schedules on FFHQ $64 \times 64$}
\label{table:FFHQ-FID}
\resizebox{\textwidth}{!}{%
\begin{tabular}{llccc}
  \toprule
  \multirow{2}{*}{Solver} & \multirow{2}{*}{Schedule} 
    & \multicolumn{3}{c}{FID $\downarrow$} \\
  \cmidrule(lr){3-5}
    & & {NFE=16} & {NFE=32} & {NFE=64} \\
  \midrule
  \multirow{3}{*}{Stochastic DDIM} & EDM 
       & \text{40.48} & \text{21.63} & \text{10.62} \\
       \cmidrule(lr){2-5}
       & Rescaled Entropy 
       & \text{30.81} & \text{14.89} & \text{7.60} \\
       \cmidrule(lr){2-5}
       & Spectral Rescaled Entropy 
       & \textbf{30.61} & \textbf{14.60} & \textbf{7.33} \\
  \midrule
  \multirow{3}{*}{Deterministic DDIM} & EDM 
       & \text{11.13} & \text{5.41} & \text{3.45} \\ 
       \cmidrule(lr){2-5}
       & Rescaled Entropy 
       & \textbf{8.10} & \textbf{4.28} & \text{3.16} \\
       \cmidrule(lr){2-5}
       & Spectral Rescaled Entropy 
       & \textbf{8.10} & \textbf{4.28} & \textbf{3.14} \\
  \bottomrule
\end{tabular}%
}
\end{minipage}
\noindent
vided in \citet{karras2024analyzing}, $103.39$, obtained using Heun solver. We observed that the entropic time produced unrecognizable images (see Appendix \ref{appendix: experiments}), therefore, we have not included it in the results. The difference between spectral rescaled entropy and rescaled entropy is small but noticeable for stochastic DDIM, with spectral rescaled entropy performing better. In contrast, for deterministic DDIM, the difference is negligible. Results on CIFAR10 and FFHQ, together with more examples of generated images, are given in appendix \ref{appendix: experiments}.

\begin{table}[htbp]
  \centering
  \caption{FID and FD-DINOv2 scores for different sampling schedules for ImageNet-$512$}
  \label{table: ImageNet-512 results}
  \resizebox{\textwidth}{!}{%
  \begin{tabular}{lll
                  S[table-format=2.2] S[table-format=2.2] S[table-format=2.2]
                  S[table-format=3.2] S[table-format=3.2] S[table-format=3.2]}
    \toprule
    \multirow{2}{*}{Solver} & \multirow{2}{*}{Network} & \multirow{2}{*}{Schedule} 
      & \multicolumn{3}{c}{FID $\downarrow$} 
      & \multicolumn{3}{c}{FD-DINOv2 $\downarrow$} \\
    \cmidrule(lr){4-6} \cmidrule(lr){7-9}
      & & & {NFE=16} & {NFE=32} & {NFE=64} & {NFE=16} & {NFE=32} & {NFE=64} \\
        \midrule
    \multirow{4}{*}{Stochastic DDIM} & \multirow{2}{*}{EDM2-XS} & EDM 
         & \text{32.31} & \text{10.01} & \text{4.98} & \text{294.25} & \text{149.91} & \text{107.00} \\ \cmidrule(lr){3-9}
         & & Rescaled Entropy 
         & \textbf{13.64} & \textbf{4.98} & \textbf{3.80} & \textbf{182.11} & \textbf{109.68} & \textbf{97.10} \\
    \cmidrule(lr){2-9}
    & \multirow{2}{*}{EDM2-XXL} & EDM 
         & \text{30.39} & \text{8.80} & \text{3.81} & \text{218.10} & \text{95.21} & \text{60.79} \\ \cmidrule(lr){3-9}
         & & Rescaled Entropy 
         & \textbf{13.38} & \textbf{3.83} & \textbf{2.60} & \textbf{108.16} & \textbf{57.05} & \textbf{46.75} \\
    \midrule
    \multirow{4}{*}{Deterministic DDIM} & \multirow{2}{*}{EDM2-XS} & EDM 
         & \text{10.42} & \text{4.81} & \text{3.83} & \textbf{156.46} & \textbf{115.94}  & \text{107.05} \\ \cmidrule(lr){3-9}
         & & Rescaled Entropy 
         & \textbf{7.57} & \textbf{4.44} & \textbf{3.75} & \text{157.32} & \text{116.52}  & \textbf{106.84} \\
    \cmidrule(lr){2-9}
    & \multirow{2}{*}{EDM2-XXL} & EDM 
         & \text{9.68} & \text{3.47} & \text{2.41} & \text{79.56} & \text{52.60}  & \text{46.27} \\ \cmidrule(lr){3-9}
         & & Rescaled Entropy 
         & \textbf{5.91} & \textbf{2.78} & \textbf{2.14} & \textbf{68.36} & \textbf{48.26}  & \textbf{43.83} \\
    \bottomrule
  \end{tabular}%
  }
\end{table}

\subsection{Limitations} \label{section: Limitations}
While our results show a clear benefit of using the entropic schedules across a wide range of datasets and fast-sampling methods, we note that these benefits are observed specifically for the first-order solvers. We tried second-order solvers as introduced in \citet{karras2022elucidating, lu2022fast} but noticed worse results (compared to the EDM schedule). We believe this is due to the use of an inappropriate information transfer function, eq. \ref{eq: information transfer}. Specifically, the definition we use considers only the current time point, whereas second-order solvers also take into account the future time point when predicting the updated state, thereby altering the entropy rate. As a result, the mismatch in temporal perspective may lead to suboptimal performance for higher-order methods as their entropy curves probably need to be readjusted based on the features of the solver.

\begin{figure}[htbp]
    \centering
    \begin{subfigure}[b]{0.45\textwidth}
        \centering
        \includegraphics[width=\textwidth]{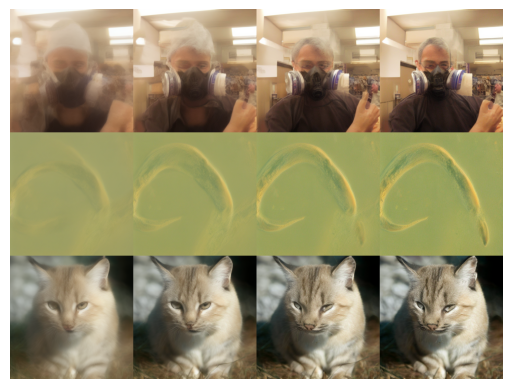}
        \caption{EDM}
    \end{subfigure}
    \hfill 
    \begin{subfigure}[b]{0.45\textwidth}
        \centering
        \includegraphics[width=\textwidth]{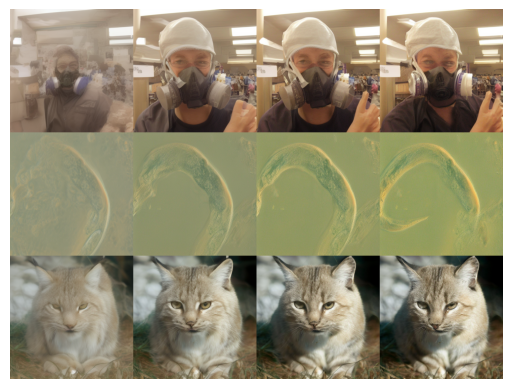}
        \caption{Rescaled Entropy}
    \end{subfigure}
    \caption{Comparison of generated images using EDM and rescaled entropic schedules with the same random seed. Images were generated using deterministic DDIM with NFE = 8, 16, 32, and 64.}
    \label{fig: example of generated images}
\end{figure}

\section{Conclusions and future work} \label{section: Conclusions}

Several avenues for future work remain open. We conjecture that using conditional entropy provides an optimal schedule for discrete generative tasks, although we currently lack a theoretical proof; nonetheless, our toy examples showed great promise. Empirically aligning entropic time with discrete diffusion models, potentially in the spirit of time warping techniques such as \citet{dieleman2022continuous}, is an exciting direction. Beyond this, entropic time may also offer a principled framework for training and model compression: in distillation, entropic time could identify the most informative stages for supervision and reduce redundancy in transferring knowledge from teacher to student models (similarly could be done for consistency models). More broadly, we propose entropic time as a candidate training schedule, enabling learning that is directly aligned with information flow. We are encouraged by the unexpected connection between our formulation and the continuous-time variational objective of \citet{kingma2021variational}. We envision that this perspective could eventually replace the heavy dataset-specific optimization required in approaches such as EDM \citep{karras2022elucidating}, leading to more efficient and adaptive training across diverse modalities, including medical imaging, audio, and text. Lastly, we note that second-order solvers, which incorporate lookahead steps, may require a fundamentally different definition of information transfer. Developing entropic analogues tailored to such solvers is another important direction for extending this framework.

\newpage
\clearpage

\bibliography{main}
\bibliographystyle{iclr2025_delta}


\newpage
\section*{NeurIPS Paper Checklist}

\begin{enumerate}

\item {\bf Claims}
    \item[] Question: Do the main claims made in the abstract and introduction accurately reflect the paper's contributions and scope?
    \item[] Answer: \answerYes{} 
    \item[] Justification:
    \item[] Guidelines:
    \begin{itemize}
        \item The answer NA means that the abstract and introduction do not include the claims made in the paper.
        \item The abstract and/or introduction should clearly state the claims made, including the contributions made in the paper and important assumptions and limitations. A No or NA answer to this question will not be perceived well by the reviewers. 
        \item The claims made should match theoretical and experimental results, and reflect how much the results can be expected to generalize to other settings. 
        \item It is fine to include aspirational goals as motivation as long as it is clear that these goals are not attained by the paper. 
    \end{itemize}

\item {\bf Limitations}
    \item[] Question: Does the paper discuss the limitations of the work performed by the authors?
    \item[] Answer: \answerYes{} 
    \item[] Justification: See sections \ref{section: Limitations} and \ref{section: Conclusions}.
    \item[] Guidelines:
    \begin{itemize}
        \item The answer NA means that the paper has no limitation while the answer No means that the paper has limitations, but those are not discussed in the paper. 
        \item The authors are encouraged to create a separate "Limitations" section in their paper.
        \item The paper should point out any strong assumptions and how robust the results are to violations of these assumptions (e.g., independence assumptions, noiseless settings, model well-specification, asymptotic approximations only holding locally). The authors should reflect on how these assumptions might be violated in practice and what the implications would be.
        \item The authors should reflect on the scope of the claims made, e.g., if the approach was only tested on a few datasets or with a few runs. In general, empirical results often depend on implicit assumptions, which should be articulated.
        \item The authors should reflect on the factors that influence the performance of the approach. For example, a facial recognition algorithm may perform poorly when image resolution is low or images are taken in low lighting. Or a speech-to-text system might not be used reliably to provide closed captions for online lectures because it fails to handle technical jargon.
        \item The authors should discuss the computational efficiency of the proposed algorithms and how they scale with dataset size.
        \item If applicable, the authors should discuss possible limitations of their approach to address problems of privacy and fairness.
        \item While the authors might fear that complete honesty about limitations might be used by reviewers as grounds for rejection, a worse outcome might be that reviewers discover limitations that aren't acknowledged in the paper. The authors should use their best judgment and recognize that individual actions in favor of transparency play an important role in developing norms that preserve the integrity of the community. Reviewers will be specifically instructed to not penalize honesty concerning limitations.
    \end{itemize}

\item {\bf Theory assumptions and proofs}
    \item[] Question: For each theoretical result, does the paper provide the full set of assumptions and a complete (and correct) proof?
    \item[] Answer: \answerYes{} 
    \item[] Justification:
    \item[] Guidelines:
    \begin{itemize}
        \item The answer NA means that the paper does not include theoretical results. 
        \item All the theorems, formulas, and proofs in the paper should be numbered and cross-referenced.
        \item All assumptions should be clearly stated or referenced in the statement of any theorems.
        \item The proofs can either appear in the main paper or the supplemental material, but if they appear in the supplemental material, the authors are encouraged to provide a short proof sketch to provide intuition. 
        \item Inversely, any informal proof provided in the core of the paper should be complemented by formal proofs provided in appendix or supplemental material.
        \item Theorems and Lemmas that the proof relies upon should be properly referenced. 
    \end{itemize}

    \item {\bf Experimental result reproducibility}
    \item[] Question: Does the paper fully disclose all the information needed to reproduce the main experimental results of the paper to the extent that it affects the main claims and/or conclusions of the paper (regardless of whether the code and data are provided or not)?
    \item[] Answer: \answerYes{} 
    \item[] Justification: We believe that the information provided in the paper should suffice for reproducing the main results in the paper. Furthermore, our code is made available. 
    \item[] Guidelines:
    \begin{itemize}
        \item The answer NA means that the paper does not include experiments.
        \item If the paper includes experiments, a No answer to this question will not be perceived well by the reviewers: Making the paper reproducible is important, regardless of whether the code and data are provided or not.
        \item If the contribution is a dataset and/or model, the authors should describe the steps taken to make their results reproducible or verifiable. 
        \item Depending on the contribution, reproducibility can be accomplished in various ways. For example, if the contribution is a novel architecture, describing the architecture fully might suffice, or if the contribution is a specific model and empirical evaluation, it may be necessary to either make it possible for others to replicate the model with the same dataset, or provide access to the model. In general. releasing code and data is often one good way to accomplish this, but reproducibility can also be provided via detailed instructions for how to replicate the results, access to a hosted model (e.g., in the case of a large language model), releasing of a model checkpoint, or other means that are appropriate to the research performed.
        \item While NeurIPS does not require releasing code, the conference does require all submissions to provide some reasonable avenue for reproducibility, which may depend on the nature of the contribution. For example
        \begin{enumerate}
            \item If the contribution is primarily a new algorithm, the paper should make it clear how to reproduce that algorithm.
            \item If the contribution is primarily a new model architecture, the paper should describe the architecture clearly and fully.
            \item If the contribution is a new model (e.g., a large language model), then there should either be a way to access this model for reproducing the results or a way to reproduce the model (e.g., with an open-source dataset or instructions for how to construct the dataset).
            \item We recognize that reproducibility may be tricky in some cases, in which case authors are welcome to describe the particular way they provide for reproducibility. In the case of closed-source models, it may be that access to the model is limited in some way (e.g., to registered users), but it should be possible for other researchers to have some path to reproducing or verifying the results.
        \end{enumerate}
    \end{itemize}

\item {\bf Open access to data and code}
    \item[] Question: Does the paper provide open access to the data and code, with sufficient instructions to faithfully reproduce the main experimental results, as described in supplemental material?
    \item[] Answer: \answerYes{} 
    \item[] Justification: Link to our GitHub repo is provided in the abstract.
    \item[] Guidelines:
    \begin{itemize}
        \item The answer NA means that paper does not include experiments requiring code.
        \item Please see the NeurIPS code and data submission guidelines (\url{https://nips.cc/public/guides/CodeSubmissionPolicy}) for more details.
        \item While we encourage the release of code and data, we understand that this might not be possible, so “No” is an acceptable answer. Papers cannot be rejected simply for not including code, unless this is central to the contribution (e.g., for a new open-source benchmark).
        \item The instructions should contain the exact command and environment needed to run to reproduce the results. See the NeurIPS code and data submission guidelines (\url{https://nips.cc/public/guides/CodeSubmissionPolicy}) for more details.
        \item The authors should provide instructions on data access and preparation, including how to access the raw data, preprocessed data, intermediate data, and generated data, etc.
        \item The authors should provide scripts to reproduce all experimental results for the new proposed method and baselines. If only a subset of experiments are reproducible, they should state which ones are omitted from the script and why.
        \item At submission time, to preserve anonymity, the authors should release anonymized versions (if applicable).
        \item Providing as much information as possible in supplemental material (appended to the paper) is recommended, but including URLs to data and code is permitted.
    \end{itemize}

\item {\bf Experimental setting/details}
    \item[] Question: Does the paper specify all the training and test details (e.g., data splits, hyperparameters, how they were chosen, type of optimizer, etc.) necessary to understand the results?
    \item[] Answer: \answerYes{} 
    \item[] Justification: For the main results of the paper, we believe that all the relevant information can be found in the paper. See Appendix \ref{appendix: experiments}.
    \item[] Guidelines:
    \begin{itemize}
        \item The answer NA means that the paper does not include experiments.
        \item The experimental setting should be presented in the core of the paper to a level of detail that is necessary to appreciate the results and make sense of them.
        \item The full details can be provided either with the code, in appendix, or as supplemental material.
    \end{itemize}

\item {\bf Experiment statistical significance}
    \item[] Question: Does the paper report error bars suitably and correctly defined or other appropriate information about the statistical significance of the experiments?
    \item[] Answer: \answerNo{} 
    \item[] Justification: All the relevant metrics for images were calculated only once and that number is reported. 
    \item[] Guidelines:
    \begin{itemize}
        \item The answer NA means that the paper does not include experiments.
        \item The authors should answer "Yes" if the results are accompanied by error bars, confidence intervals, or statistical significance tests, at least for the experiments that support the main claims of the paper.
        \item The factors of variability that the error bars are capturing should be clearly stated (for example, train/test split, initialization, random drawing of some parameter, or overall run with given experimental conditions).
        \item The method for calculating the error bars should be explained (closed form formula, call to a library function, bootstrap, etc.)
        \item The assumptions made should be given (e.g., Normally distributed errors).
        \item It should be clear whether the error bar is the standard deviation or the standard error of the mean.
        \item It is OK to report 1-sigma error bars, but one should state it. The authors should preferably report a 2-sigma error bar than state that they have a 96\% CI, if the hypothesis of Normality of errors is not verified.
        \item For asymmetric distributions, the authors should be careful not to show in tables or figures symmetric error bars that would yield results that are out of range (e.g. negative error rates).
        \item If error bars are reported in tables or plots, The authors should explain in the text how they were calculated and reference the corresponding figures or tables in the text.
    \end{itemize}

\item {\bf Experiments compute resources}
    \item[] Question: For each experiment, does the paper provide sufficient information on the computer resources (type of compute workers, memory, time of execution) needed to reproduce the experiments?
    \item[] Answer: \answerNo{} 
    \item[] Justification:
    \item[] Guidelines:
    \begin{itemize}
        \item The answer NA means that the paper does not include experiments.
        \item The paper should indicate the type of compute workers CPU or GPU, internal cluster, or cloud provider, including relevant memory and storage.
        \item The paper should provide the amount of compute required for each of the individual experimental runs as well as estimate the total compute. 
        \item The paper should disclose whether the full research project required more compute than the experiments reported in the paper (e.g., preliminary or failed experiments that didn't make it into the paper). 
    \end{itemize}
    
\item {\bf Code of ethics}
    \item[] Question: Does the research conducted in the paper conform, in every respect, with the NeurIPS Code of Ethics \url{https://neurips.cc/public/EthicsGuidelines}?
    \item[] Answer: \answerYes{} 
    \item[] Justification:
    \item[] Guidelines:
    \begin{itemize}
        \item The answer NA means that the authors have not reviewed the NeurIPS Code of Ethics.
        \item If the authors answer No, they should explain the special circumstances that require a deviation from the Code of Ethics.
        \item The authors should make sure to preserve anonymity (e.g., if there is a special consideration due to laws or regulations in their jurisdiction).
    \end{itemize}

\item {\bf Broader impacts}
    \item[] Question: Does the paper discuss both potential positive societal impacts and negative societal impacts of the work performed?
    \item[] Answer: \answerYes{} 
    \item[] Justification: See appendix \ref{appendix: broader impact}.
    \item[] Guidelines:
    \begin{itemize}
        \item The answer NA means that there is no societal impact of the work performed.
        \item If the authors answer NA or No, they should explain why their work has no societal impact or why the paper does not address societal impact.
        \item Examples of negative societal impacts include potential malicious or unintended uses (e.g., disinformation, generating fake profiles, surveillance), fairness considerations (e.g., deployment of technologies that could make decisions that unfairly impact specific groups), privacy considerations, and security considerations.
        \item The conference expects that many papers will be foundational research and not tied to particular applications, let alone deployments. However, if there is a direct path to any negative applications, the authors should point it out. For example, it is legitimate to point out that an improvement in the quality of generative models could be used to generate deepfakes for disinformation. On the other hand, it is not needed to point out that a generic algorithm for optimizing neural networks could enable people to train models that generate Deepfakes faster.
        \item The authors should consider possible harms that could arise when the technology is being used as intended and functioning correctly, harms that could arise when the technology is being used as intended but gives incorrect results, and harms following from (intentional or unintentional) misuse of the technology.
        \item If there are negative societal impacts, the authors could also discuss possible mitigation strategies (e.g., gated release of models, providing defenses in addition to attacks, mechanisms for monitoring misuse, mechanisms to monitor how a system learns from feedback over time, improving the efficiency and accessibility of ML).
    \end{itemize}
    
\item {\bf Safeguards}
    \item[] Question: Does the paper describe safeguards that have been put in place for responsible release of data or models that have a high risk for misuse (e.g., pretrained language models, image generators, or scraped datasets)?
    \item[] Answer: \answerNA{} 
    \item[] Justification:
    \item[] Guidelines:
    \begin{itemize}
        \item The answer NA means that the paper poses no such risks.
        \item Released models that have a high risk for misuse or dual-use should be released with necessary safeguards to allow for controlled use of the model, for example by requiring that users adhere to usage guidelines or restrictions to access the model or implementing safety filters. 
        \item Datasets that have been scraped from the Internet could pose safety risks. The authors should describe how they avoided releasing unsafe images.
        \item We recognize that providing effective safeguards is challenging, and many papers do not require this, but we encourage authors to take this into account and make a best faith effort.
    \end{itemize}

\item {\bf Licenses for existing assets}
    \item[] Question: Are the creators or original owners of assets (e.g., code, data, models), used in the paper, properly credited and are the license and terms of use explicitly mentioned and properly respected?
    \item[] Answer: \answerYes{} 
    \item[] Justification:
    \item[] Guidelines:
    \begin{itemize}
        \item The answer NA means that the paper does not use existing assets.
        \item The authors should cite the original paper that produced the code package or dataset.
        \item The authors should state which version of the asset is used and, if possible, include a URL.
        \item The name of the license (e.g., CC-BY 4.0) should be included for each asset.
        \item For scraped data from a particular source (e.g., website), the copyright and terms of service of that source should be provided.
        \item If assets are released, the license, copyright information, and terms of use in the package should be provided. For popular datasets, \url{paperswithcode.com/datasets} has curated licenses for some datasets. Their licensing guide can help determine the license of a dataset.
        \item For existing datasets that are re-packaged, both the original license and the license of the derived asset (if it has changed) should be provided.
        \item If this information is not available online, the authors are encouraged to reach out to the asset's creators.
    \end{itemize}

\item {\bf New assets}
    \item[] Question: Are new assets introduced in the paper well documented and is the documentation provided alongside the assets?
    \item[] Answer: \answerNA{} 
    \item[] Justification:
    \item[] Guidelines:
    \begin{itemize}
        \item The answer NA means that the paper does not release new assets.
        \item Researchers should communicate the details of the dataset/code/model as part of their submissions via structured templates. This includes details about training, license, limitations, etc. 
        \item The paper should discuss whether and how consent was obtained from people whose asset is used.
        \item At submission time, remember to anonymize your assets (if applicable). You can either create an anonymized URL or include an anonymized zip file.
    \end{itemize}

\item {\bf Crowdsourcing and research with human subjects}
    \item[] Question: For crowdsourcing experiments and research with human subjects, does the paper include the full text of instructions given to participants and screenshots, if applicable, as well as details about compensation (if any)? 
    \item[] Answer: \answerNA{} 
    \item[] Justification: 
    \item[] Guidelines:
    \begin{itemize}
        \item The answer NA means that the paper does not involve crowdsourcing nor research with human subjects.
        \item Including this information in the supplemental material is fine, but if the main contribution of the paper involves human subjects, then as much detail as possible should be included in the main paper. 
        \item According to the NeurIPS Code of Ethics, workers involved in data collection, curation, or other labor should be paid at least the minimum wage in the country of the data collector. 
    \end{itemize}

\item {\bf Institutional review board (IRB) approvals or equivalent for research with human subjects}
    \item[] Question: Does the paper describe potential risks incurred by study participants, whether such risks were disclosed to the subjects, and whether Institutional Review Board (IRB) approvals (or an equivalent approval/review based on the requirements of your country or institution) were obtained?
    \item[] Answer: \answerNA{} 
    \item[] Justification: 
    \item[] Guidelines:
    \begin{itemize}
        \item The answer NA means that the paper does not involve crowdsourcing nor research with human subjects.
        \item Depending on the country in which research is conducted, IRB approval (or equivalent) may be required for any human subjects research. If you obtained IRB approval, you should clearly state this in the paper. 
        \item We recognize that the procedures for this may vary significantly between institutions and locations, and we expect authors to adhere to the NeurIPS Code of Ethics and the guidelines for their institution. 
        \item For initial submissions, do not include any information that would break anonymity (if applicable), such as the institution conducting the review.
    \end{itemize}

\item {\bf Declaration of LLM usage}
    \item[] Question: Does the paper describe the usage of LLMs if it is an important, original, or non-standard component of the core methods in this research? Note that if the LLM is used only for writing, editing, or formatting purposes and does not impact the core methodology, scientific rigorousness, or originality of the research, declaration is not required.
    \item[] Answer: \answerNA{} 
    \item[] Justification: 
    \item[] Guidelines:
    \begin{itemize}
        \item The answer NA means that the core method development in this research does not involve LLMs as any important, original, or non-standard components.
        \item Please refer to our LLM policy (\url{https://neurips.cc/Conferences/2025/LLM}) for what should or should not be described.
    \end{itemize}

\end{enumerate}

\newpage

\appendix

\section{Broader impact} \label{appendix: broader impact}

This work proposes entropic and rescaled entropic time schedules for generative diffusion models, improving performance, especially in low NFE regimes, without additional training overhead. Our approach may benefit applications relying on fast generation, such as medical imaging and digital content creation. Furthermore, our work contributes to a deeper understanding of the relationship between information theory and generative modeling, which may inspire further theoretical advancements.

As with all generative models, there is potential for misuse in generating synthetic media that could be used for disinformation or impersonation. Improvements in efficiency may increase such risks by lowering the barrier to large-scale generation.

\section{Entropy Production} \label{section: entropy production}

Here, we show
\begin{equation}
    \dot{\mathbf{H}}[\rvx_t] = \mathbb{E}_{p_t(x_t)} [\nabla (f)] + \frac{g^2_t}{2} \mathbb{E}_{p_t(x_t)} [|| \nabla \log p(\rvx_t) ||^2 ].
\end{equation}
By looking inside the integral of $\dot{\mathbf{H}}[\rvx_t]$, we get
\begin{equation}
\begin{aligned}
    \dot{\mathbf{H}}[\rvx_t] & = - \int \left( \dot{p}(x_t) \log p(x_t) + p(x_t) \frac{\dot{p}(x_t)}{p(x_t)} \right) dx_t \\
    & = - \int \dot{p}(x_t) \log p(x_t) dx_t - \frac{d}{dt} \int p(x_t) dx_t \\
    & = - \int \dot{p}(x_t) \log p(x_t) dx_t.
\end{aligned}
\end{equation}
Assuming our dynamic is determined by the SDE \ref{eq: SDE1}, we can use the Fokker-Planck equation to simplify the derivative as follows
\begin{equation}
\begin{aligned}
    \dot{\mathbf{H}}[\rvx_t] & = - \int \left(- \nabla \left( \left( f_t - \frac{g^2_t}{2} \nabla \log p(x_t) \right) p(x_t) \right) \right) \log p(x_t) dx_t \\
    & = - \int \Inner{f_t - \frac{g^2_t}{2} \nabla \log p(x_t)}{\nabla \log p(x_t)} p(x_t) dx_t \\
    & = - \int \inner{f_t}{\nabla \log p(x_t)} p(x_t) dx_t + \int \frac{g^2_t}{2} \inner{\nabla \log p(x_t)}{\nabla \log p(x_t)} p(x_t) dx_t \\
    & = - \int \inner{f_t}{\nabla p(x_t)} dx_t + \int \frac{g^2_t}{2} || \nabla \log p(x_t) ||^2 p(x_t) dx_t \\
    & =  \int \nabla (f_t) p(x_t) dx_t + \frac{g^2_t}{2} \mathbb{E}_{p_t(x_t)} [|| \nabla \log p(\rvx_t) ||^2 ] \\
    & =  \mathbb{E}_{p_t(x_t)} [\nabla (f_t)] + \frac{g^2_t}{2} \mathbb{E}_{p_t(x_t)} [|| \nabla \log p(\rvx_t) ||^2 ]
\end{aligned}
\end{equation}
which is exactly what we wanted to show. We used integration by parts in going from the first line to the second and from the fourth to the fifth.

\section{Optimal schedule for deterministic DDIM} \label{appendix: DDDIM schedule}

Here, we show that the optimal schedule for the deterministic DDIM \citep{song2022denoisingdiffusionimplicitmodels} for the EDM SDE is the same as the one given in \citet{sabour2024align}.

\begin{theorem}
    Let $p_{\text{data}}(\mathbf{x}) = \mathcal{N}(0, c^2 \mathbf{I})$ and the diffusion process is given by the EDM SDE. Sample $\rvx_{t_{\max}} \sim p(\mathbf{x}, t_{\max})$ and use the deterministic DDIM using $n$ steps along the schedule
    \[
    t_{\max} = t_n > t_{n-1} > \cdots > t_1 > t_0 = t_{\min}
    \]
    to obtain $\rvx_{t_{\min}}$. The optimal schedule $t^*$ minimizing the KL-divergence between $p(\rvx, t_{\min})$ and the distribution of $\rvx_{t_{\min}}$ is given by
    \[
    \quad t^*_i = c \tan \left( \left(1 - \frac{i}{n} \right) \alpha_{\min} + \frac{i}{n} \alpha_{\max} \right)
    \]
    where
    \[
    \alpha_{\min} := \arctan(t_{\min}/c), \quad \alpha_{\max} := \arctan(t_{\max}/c).
    \]
\end{theorem}
\begin{proof}
    The deterministic DDIM update is given as
    \[
    \rvx_{t_{i-1}} = \hat{x}_0(\rvx_{t_i}) + \frac{t_{i_{i-1}}}{t_i} \left( \rvx_{t_i} - \hat{x}_0(\rvx_{t_i}) \right)
    \]
    where $\hat{x}_0(\rvx_{t_i}) = \rvx_{t} + t^2_i \nabla \log p(\rvx_{t_i})$. By using an analytical expression for the score, we get a simplified expression for the update
    \[
    \rvx_{t_{i-1}} = \frac{c^2+t_{i-1}t_i}{c^2+t^2_i} \rvx_{t_i}.
    \]
    This turns out to be exactly the same update as in ODE Euler method (equation 18 in \citet{sabour2024align}). Therefore, our claim follows from the proof of theorem 3.1. in \citet{sabour2024align}.
\end{proof}

\section{Proofs from section \ref{section: isoentropic  time}} \label{section: proofs}

\begin{theorem}
    Given an SDE and initial data distribution $p_0(x)$, $\phi(t) = \mathbf{H}[x_0|x_t]$ and $\phi(t) = \int_0^t ds \sigma_s \dot{\mathbf{H}}[\rvx_0|\rvx_s]$ are proper time changes.
\end{theorem}

\begin{proof}
    As already mentioned, a proper time change must be a strictly increasing, continuous function. Since $\mathbf{H}[x_0|x_t]$ has a derivative (see section \ref{section: Estimating a conditional entropy}), we need to show that it is positive. However, our claim follows from equation \ref{eq: connection between quad_error and conditional entropy production} (the squared error is equal to zero only when an initial distribution consists of one data point).
\end{proof}

\begin{theorem}
    Given two SDEs as given in definition \ref{definition: equivalence up to a time change}, and following time changes
    \begin{enumerate}
    \item $ \phi: t \mapsto s=f(t) $
    \item $ \Phi_t: t \mapsto \mathbf{H}_t[\rvx_0|\rvx_t] $
    \item $ \Phi_s: s \mapsto \mathbf{H}_s[\rvx_0|\rvx_s] $,
\end{enumerate}
    it follows that
    $$ F := \Phi_s \circ \phi \circ \Phi_t^{-1}: \mathbf{H}_t[\rvx_0|\rvx_t] \mapsto \mathbf{H}_s[\rvx_0|\rvx_s] $$
    is a proper time change implementing the equivalence and is equal to the identity map, $F = id$.
\end{theorem}

\begin{proof}
Immediately, we can see that $g$ is a proper time change since it is composed of other time changes. Similarly, using a chain rule, it is observed that $g$ implements the equivalence. Furthermore,
\begin{equation}
\begin{aligned}
    F(\mathbf{H}_t[\rvx_0|\rvx_t]) & = (\Phi_s \circ \phi) (\Phi_t^{-1}(\mathbf{H}_t[\rvx_0|\rvx_t])) = \Phi_s(\phi(t)) = \mathbf{H}_s[\rvx_0|\rvx_{\phi(t)}].
\end{aligned}
\end{equation}
However, since $p_t(x) = q_{\phi(t)}(x)$ and $p(x_t|x_0) = q(x_{\phi(t)}|x_0)$, it follows
\begin{equation}
\begin{aligned}
    \mathbf{H}_t[\rvx_0|\rvx_t] & = - \iint p(x_t, x_0) \ln{(p(x_0|x_t))} dx_0 dx_t \\
    & = - \iint q(x_{\phi(t)}, x_0) \ln{(q(x_0|x_{\phi(t)}))} dx_0 dx_{\phi(t)} = \mathbf{H}_s[\rvx_0|\rvx_{\phi(t)}],
\end{aligned}
\end{equation}
where $x_t = x_{\phi(t)}$ (i.e. are the same spatial point) and time subscripts represent at which point in time the probability distribution is evaluated. This proves that $F = id$.
\end{proof}

Similarly, we can prove the same claim for the rescaled entropic time since $\sigma(t) = \sigma(\phi(t))$ for any proper change of time $\phi$.

\section{Rescaled entropy for Gaussian data}
\label{appendix: rescaled entropy}

Here, we show that, in the case of the EDM noise schedule, the rescaled entropic time is the optimal sampling schedule for the ODE flow when data comes from a normal distribution with variance $c^2$ (equation \ref{eq: optimal schedule}).

Recall the expression for the rescaled entropy, $\int_0^t \sigma(\tau) \dot{\mathbf{H}}[\rvx_0 | \rvx_{\tau}] d\tau$. From equation \ref{eq: EDM conditional entropy}, we have
\begin{equation}
\begin{aligned}
    \int_0^t \sigma(\tau) \dot{\mathbf{H}}[\rvx_0 | \rvx_{\tau}] d\tau & = \int_0^t \left( D \dot{\sigma}(\tau) - s(\tau)^2 \dot{\sigma}(\tau) \sigma(\tau)^2 \mathbb{E}_{p_{\tau}(x_{\tau})} [|| \nabla \log p_{\tau}(\rvx_{\tau}) ||^2 ] \right) d\tau.
\end{aligned}
\end{equation}
Using the facts that $\sigma(\tau) = \tau$, $s(\tau) = 1$ and $\nabla \log p_{\tau}(x_{\tau}) = \frac{-x_{\tau}}{s(\tau)^2\sigma(\tau)^2+s(\tau)^2c^2}$, we get
\begin{equation}
\begin{aligned}
    \int_0^t \sigma(\tau) \dot{\mathbf{H}}[\rvx_0 | \rvx_{\tau}] d\tau & = \int_0^t \left( D - \tau^2 \frac{D}{\tau^2+c^2} \right) d\tau \\
    & = \int_0^t \frac{D c^2}{\tau^2+c^2} d\tau = D c \arctan\left( \frac{t}{c} \right).
\end{aligned}
\end{equation}
Therefore, a linear sampling schedule, $[t_{min}, t_1, ... t_{max}]$, in the rescaled entropic time is given by
\begin{equation}
     D c \arctan\left( \frac{t_i}{c} \right) = D c \left( \alpha_{min} + \frac{i}{N} (\alpha_{max} - \alpha_{min}) \right)
\end{equation}
where $\alpha_{min/max} = \arctan \left( \frac{t_{min/max}}{c} \right)$. Exactly the same as equation \ref{eq: optimal schedule}.

\section{Connection with a squared error and loss} \label{section: connection with a quadratic error and loss}

In this Appendix, we show connections between conditional entropy production and some commonly used expressions in the diffusion literature. Firstly, we show how the conditional entropy production is related to the squared error at time $t$, $\epsilon_t^2$.

\begin{equation}
\begin{aligned}
    \epsilon_t^2 & = \mathbb{E}_{p(x_0,x_t)} [||\rvx_0 - \hat{\rvx}_0(\rvx_t)||^2] = \iint ||x_0 - \hat{x}_0(x_t)||^2 p(x_t|x_0) p(x_0) dx_t dx_0 \\
    & = \iint \left|\left| x_0 - \frac{(x_t + s(t)^2 \sigma(t)^2 \nabla(\log p(x_t))}{s(t)} \right|\right|^2 p(x_t|x_0) p(x_0) dx_t dx_0
\end{aligned}
\end{equation}

The squared error encapsulates our uncertainty at time $t$ about the final sample $x_0$. The following simplification of the above equation gives a more precise meaning.

\begin{equation}
\begin{aligned}
    \epsilon_t^2 & = \mathbb{E}_{p_t(x_t)} [ \mathbb{E}_{p(x_0 | x_t)}[||\rvx_0 - \mathbb{E}_{p(y_0 | x_t)}[\rvy_0]||^2] ] \\
    & = \mathbb{E}_{p_t(x_t)} [tr(\sigma_{\rvx_0|\rvx_t}^2)].
\end{aligned}
\end{equation}

From Appendix \ref{section: Tweedie second order}, we know

\begin{equation}
    Var_{p(x_0|x_t)}[\rvx_0] = \sigma(t)^2 (I + s(t)^2 \sigma(t)^2 H[\log p_t(x_t)]).
\end{equation}
Hence, 

\begin{equation}
\begin{aligned}
    \epsilon_t^2 = \mathbb{E}_{p_t(x_t)} [tr(\sigma_{\rvx_0 | \rvx_t}^2)] & = \sigma(t)^2 \mathbb{E}_{p_t(x_t)}[ tr(I + s(t)^2 \sigma(t)^2 H[\log p(\rvx_t)]) ] \\
    & = \sigma(t)^2 ( D - s(t)^2 \sigma(t)^2 \mathbb{E}_{p_t(x_t)}[ || \nabla \log (p_t(\rvx_t)) ||^2 ] ) \\
    & = \frac{\sigma(t)^3}{\dot{\sigma}(t)} \left( \frac{D \dot{\sigma}(t)}{\sigma(t)} - s(t)^2 \dot{\sigma}(t) \sigma(t) \mathbb{E}_{p_t(x_t)}[ || \nabla \log (p_t(\rvx_t)) ||^2 ] \right) \\
    & = \frac{\sigma(t)^3}{\dot{\sigma}(t)} \dot{\mathbf{H}}[\rvx_0|\rvx_t].
\end{aligned}
\end{equation}

Following notation from \citet{karras2022elucidating} for the loss at time $t$, we have

\begin{equation}
    \mathcal{L}(t) = \mathbb{E}_{p_0(x_0), \mathcal{N}(\epsilon; 0, I)} [\lambda(t) \norm{c_{out}(t) F_{\theta} - s(t) x_0 + c_{skip}(t) (s(t) x_0 + s(t) \sigma(t) \epsilon)}^2].
\end{equation}

The formula for a prediction $\hat{x}_0(x_t)$ is given by

\begin{equation}
    \hat{x}_0(x_t) = \frac{c_{out}(t) F_{\theta}(x_t) + c_{skip}(t) x_t}{s(t)}.
\end{equation}

We can express the loss at time $t$ using the squared error as

\begin{equation}
\begin{aligned}
    \mathcal{L}(t) & = \lambda(t) \mathbb{E} [\norm{s(t) \rvx_0 - s(t) \hat{\rvx}_0(\rvx_t)}^2] = \lambda(t) s(t)^2 \epsilon_t^2.
\end{aligned}
\end{equation}

Furthermore, using the connection between a squared error and a conditional entropy production, we get

\begin{equation}
\begin{aligned}
    \dot{\mathbf{H}}[\rvx_0|\rvx_t] & = \frac{\dot{\sigma}(t)}{\lambda(t) s(t)^2 \sigma(t)^3} \mathcal{L}(t).
\end{aligned}
\end{equation}

\section{Tweedie's second-order formula} \label{section: Tweedie second order}

Assume we are given a distribution $p(y)$ that is obtained by adding a Gaussian noise to a distribution $q(x)$, i.e. $ q(y|x) = \mathcal{N}(y; s x, s^2 \sigma^2) $.

Now given some $y \sim p(y)$, if we are interested in which $x \sim q(x)$ generated it, the best we can do is guess $\hat{x}(y) = E_{q(x|y)}[x]$. Tweedie's formula gives us

\begin{equation}
    \mathbb{E}_{q(x|y)}[x] = \frac{y + s^2 \sigma^2 \nabla_y \log p(y)}{s}
\end{equation}

Now, we might ask how sure we are of our guess. To answer that question, we need to look at the variance,
$Var_{q(x|y)}[x] = \mathbb{E}_{q(x|y)}[x^2] - \mathbb{E}_{q(x|y)}[x]^2$. In this section, we derive the following result
\begin{equation}
    Var_{q(x|y)}[x] = s \sigma^2 \nabla_y E_{q(x|y)}[x] = \sigma^2 (I + s^2 \sigma^2 H[\log p(y)]).
\end{equation}
However, a more general result regarding the cumulants of $q(x|y)$ holds \citep{dytso2022conditional}. That is, all the cumulants can be calculated using the score function and its derivatives.

Since we already have $\mathbb{E}_{q(x|y)}[x]$, we need to find an expression for $\mathbb{E}_{q(x|y)}[x^2]$.

\begin{equation}
\begin{aligned}
    \mathbb{E}_{q(x|y)}[x^2] & = \int dx \frac{q(y|x) q(x)}{p(y)} x^2 = \int dx \frac{q(x) x}{p(y)} q(y|x) x \\
    & = \int dx \frac{x q(x)}{p(y)} \frac{y q(y|x) + s^2 \sigma^2 \nabla_y q(y|x)}{s} \\
    & = \frac{y \mathbb{E}_{q(x|y)}[x]}{s} + \frac{s^2 \sigma^2}{s p(y)} \nabla_y \int dx q(y|x) q(x) x
\end{aligned}
\end{equation}

Where in going from the first line to the second, we used $\nabla_y q(y|x) = \frac{sx-y}{s^2 \sigma^2} q(y|x)$. However, we seem to have encountered a problem with the second term in our expression. However, by using $q(x, y) = q(y|x) q(x) = q(x|y) p(y)$, for the second term we get

\begin{equation}
\begin{aligned}
    \nabla_y \int dx q(y|x) q(x) x & = \nabla_y \int dx q(x|y) p(y) x \\
    & = \nabla_y \left( p(y) \int dx q(x|y) x \right) = \nabla_y (p(y) \mathbb{E}_{q(x|y)}[x]) \\
    & = \mathbb{E}_{q(x|y)}[x] \nabla_y p(y) + p(y) \nabla_y \mathbb{E}_{q(x|y)}[x].
\end{aligned}
\end{equation}

Hence,

\begin{equation}
\begin{aligned}
    \mathbb{E}_{q(x|y)}[x^2] & = \frac{y \mathbb{E}_{q(x|y)}[x]}{s} + \frac{s^2 \sigma^2}{s p(y)} (\mathbb{E}_{q(x|y)}[x] \nabla_y p(y) + p(y) \nabla_y \mathbb{E}_{q(x|y)}[x]) \\
    & = \frac{y \mathbb{E}_{q(x|y)}[x]}{s} + \frac{s^2 \sigma^2 (\mathbb{E}_{q(x|y)}[x] \nabla_y \log p(y) + \nabla_y \mathbb{E}_{q(x|y)}[x])}{s} \\
    & = \mathbb{E}_{q(x|y)}[x] \frac{ y + s^2 \sigma^2 \nabla_y \log p(y) }{s} + \frac{s^2 \sigma^2}{s} \nabla_y \mathbb{E}_{q(x|y)}[x] \\
    & = E_{q(x|y)}[x]^2 + \frac{s^2 \sigma^2}{s} \nabla_y \mathbb{E}_{q(x|y)}[x].
\end{aligned}
\end{equation}

Now, we get an elegant expression for the variance

\begin{equation}
    Var_{q(x|y)}[x] = s \sigma^2 \nabla_y \mathbb{E}_{q(x|y)}[x] = \sigma^2 (1 + s^2 \sigma^2 \partial_{yy} \log p(y)).
\end{equation}

So far, we have been dealing with one-dimensional random variables, but it is easy to generalize all the steps to arbitrary dimensions, which gives us the general formula

\begin{equation}
    Var_{q(x|y)}[x] = s \sigma^2 \nabla_y E_{q(x|y)}[x] = \sigma^2 (I + s^2 \sigma^2 H[\log p(y)]).
\end{equation}

\section{Experimental details and Additional results}
\label{appendix: experiments}

\subsection{Algorithms}

Here, we present algorithm~\ref{algorithm: spectral epsilon^2}, which estimates the spectral decomposition of the squared error from the data. Having numerical values of the squared error and its spectral decomposition can be used to compute other entropic quantities of interest. A similar algorithm can be used to obtain entropy and squared error for different orthonormal basis.

\begin{algorithm}
\caption{Estimation of spectral decomposition of $\epsilon^2(t)$}
\label{algorithm: spectral epsilon^2}
\begin{algorithmic}[1]
\State \textbf{Input:} $\;D_\theta(x,\sigma),\;\text{Encoder}(x),\;\sigma(t),\;s(t),\;t_{i\in\{0,\dots,N\}},\;M$
\For{$i \in \{0,\dots,N\}$} 
    \State Sample $x_0^{\,j} \sim p_0\text{ for } j = 1,\dots,M$
    \State $\mathbf{z}_0 \;=\;\text{Encoder}(\mathbf{x}_0)$
    \State Sample $\nu_j \sim \mathcal{N}(0,I)$ for \(j = 1,\dots,M\)
    \State \(\mathbf{z}_{t_i} = s(t_i)\,\mathbf{z}_0 \;+\; s(t_i)\,\sigma(t_i)\,\nu\)
    \State $\hat{\mathbf{z}}_0 = D_\theta\bigl(\mathbf{z}_{t_i},\,\sigma(t_i)\bigr)$
    \For{$j = 1,\dots,M$}
        \State $\Delta^j \;=\; \hat{z}_0^j \;-\; z_0^j$
        \State $\widehat{\Delta}^j \;=\; \operatorname{FFT2D}\bigl(\Delta^j\bigr)$
        \State $|\widehat{\Delta}^j|^2 \;=\; \bigl|\Re(\widehat{\Delta}^j)\bigr|^2 \;+\; \bigl|\Im(\widehat{\Delta}^j)\bigr|^2$
        \Comment{Modulus squared for each matrix entry}
    \EndFor
    \State $\mathcal{E}^2(t_i) = \frac{1}{M}\sum_{j=1}^M |\widehat{\Delta}^j|^2 $
    \Comment{Matrix of squared error for different frequencies}
\EndFor
\State \textbf{output} $\mathcal{E}^2(t)$
\end{algorithmic}
\end{algorithm}


\subsection{One-dimensional experiments}

\begin{table}[htbp]
\captionsetup{type=figure}
  \centering
  \renewcommand{\arraystretch}{1.2}
  \resizebox{\textwidth}{!}{
  \begin{tabular}{
    >{\centering\arraybackslash}m{1.5cm} 
    >{\centering\arraybackslash}m{8cm} 
    >{\centering\arraybackslash}m{8cm}
  }
    &    \hspace{0.75cm} Mixture of 5 & \hspace{0.75cm} Mixture of 15 \\
    \raisebox{0.75cm}{Discrete} & 
    \includegraphics[width=\linewidth]{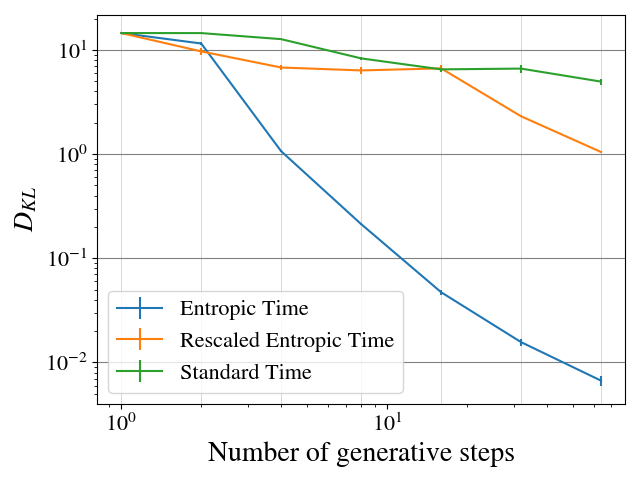} & 
    \includegraphics[width=\linewidth]{figures/1D_EDM/D_KL_for_15_starting_positions_EDM.png} \\
    \raisebox{0.75cm}{Continuous} & 
    \includegraphics[width=\linewidth]{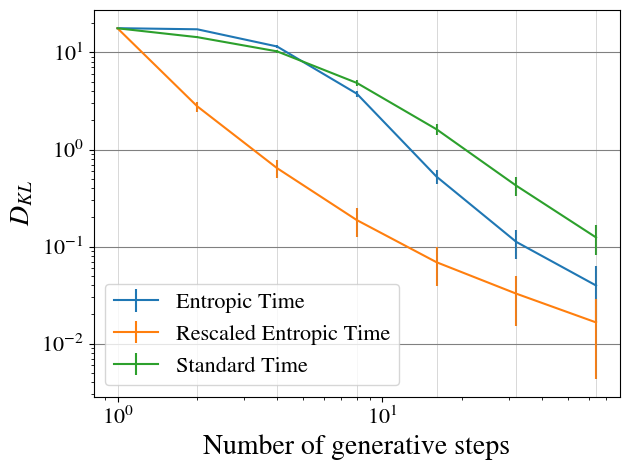} & 
    \includegraphics[width=\linewidth]{figures/1D_EDM/D_KL_EDM_15_Gaussians.png} \\
  \end{tabular}
  }
    \caption{Kullback–Leibler divergence against the number of generative steps for different time parameterizations for a mixture of data points (discrete) and Gaussians (continuous).}
    \label{fig: 1-D toy example, EDM}
\end{table}

We used an analytic expression of a score function to compare the performance of a few-step generation process in different time parameterizations in one dimension. We used equidistant steps in the standard time, entropic time, and rescaled entropic time. All entropic quantities were obtained from the squared error using equation \ref{eq: connection between quad_error and conditional entropy production}. The squared error was estimated at $10^4$ equidistant timesteps with $10^3$ samples at each timestep. We used a mixture of data points (discrete case) and a mixture of Gaussians (continuous case). In both cases, the data had a mean of zero and a standard deviation of one. For sample generation, we used the stochastic DDIM \citep{song2022denoisingdiffusionimplicitmodels}. Results are given in figure \ref{fig: 1-D toy example, EDM}.

For the discrete case, the performance was measured by creating nonoverlapping bins around data points, $[a_i-\epsilon, a_i+\epsilon]$, and calculating the Kullback-Leibler divergence between the initial distribution and the binned distribution ($p_{bin}(a_i) = $ probability of a generated sample ending up in the $i$-th bin). A variance-preserving SDE and EDM SDE were used for our experiments. Datapoints were randomly initialized and Kullback-Leibler divergence was estimated $10^2$ times using $10^4$ different paths, so the mean and variance of the KL estimate could be obtained. 

For the continuous case, the performance was measured by estimating the SDE-generated distribution using Gaussian kernel density estimation (with a standard deviation of $10^{-2}$) and then evaluating the KL divergence using Monte-Carlo methods with $10^3$ samples. Similarly to the discrete case, the KL divergence was estimated $10^2$ times using $10^4$ different paths to estimate the SDE-generated distribution.

\subsection{CIFAR10, FFHQ, ImageNet}

For ImageNet-$64$, we used EDM2-S and EDM2-L models and for ImageNet-512, we used EDM2-XS and EDM2-XXL models provided by \citet{karras2024analyzing}. For CIFAR10 and FFHQ, we used unconditional VP models provided by \citet{karras2022elucidating}. For generating samples, we used the stochastic and deterministic DDIM \citep{song2022denoisingdiffusionimplicitmodels}. To compare performance between different runs, we used the FID \citep{heusel2017gans} and, for ImageNet, FD-DINOv2 \citep{oquab2023dinov2, stein2023exposing} scores provided by the \citet{karras2024analyzing} implementation. We used FD-DINOv2 as it correlates better with human preferences \citep{stein2023exposing}. We used implemantations provided by \url{https://github.com/NVlabs/edm}, for CIFAR10 and FFHQ, and \url{https://github.com/NVlabs/edm2}, for ImageNet. We generated $50 000$ images and compared them against pre-computed reference statistics. All reported results are from a single (first) run. Class labels for ImageNet were drawn from a uniform distribution.

For all data sets, entropy and rescaled entropy were calculated using an estimation of squared error using equation \ref{eq: connection between quad_error and conditional entropy production}. The squared error was estimated at $128$ time points according to the EDM schedule ($\rho=7$, $\sigma_{min}=0.002$, $\sigma_{max}=80$) using the Monte-Carlo method with $1024$ samples at each timestep. In order to obtain (rescaled) entropy, any numerical integration technique should work. We decided on the simplest one, taking the difference in time steps, multiplying it by the derivative, and cumulatively summing it up to a time point $t$. 

We decided on 128 time points by comparing it to 512 time points for CIFAR10 and FFHQ, and noticing no perceivable difference in the final entropy curves. Also, rescaled entropy was calculated for ImageNet-64 with both network sizes, $S$ and $L$, and there was no significant difference between them, as expected (since it depends only on a forward process and the initial data distribution). Therefore, we used the smallest models to estimate entropic quantities. Regarding a spectral rescaled entropy, $10 000$ images were used to estimate amplitudes of different frequencies.

As already stated in the main text, the entropic time generated blurry images and was not used for further comparison. An example of images generated with the deterministic DDIM solver using the entropic schedule over 64 steps, with the EDM2-L model, is given in figure \ref{fig: Entropy}. Results obtained for CIFAR10 and FFHQ are given in tables \ref{table:CIFAR10-FID} and \ref{table:FFHQ-FID}, respectively. Results for ImageNet-64 are given in table \ref{table: ImageNet-64 FID}. Examples of generated images for ImageNet-64 using the EDM and rescaled entropy schedules are given in figures \ref{fig: 64_s} and \ref{fig: 64_l}. For the sake of completeness, we include FID and FD-DINOv2 scores for models optimized for FID scores and models optimized for DINO scores in tables \ref{table: 512 FID} and \ref{table: 512 DINO}, respectively. We observe some interesting behavior of these results, such as DINO-optimized models giving better FID than FID-optimized ones for NFE=16. Also, we can see that the FID score can go up while the DINO score steadily decreases for DINO-optimized models. This shows that those two metrics asses and value vastly different properties of generated images. In addition, we notice that our DINO results are comparable to the results provided in \citet{karras2024analyzing} obtained using Heun second-order solver. Figure \ref{fig: example of generated images 2} show how the number of function evaluations affect the generated images when using EDM and rescaled entropic schedules. Examples of generated images for ImageNet-512 using the EDM and rescaled entropy schedules with stochastic DDIM are given in figures \ref{fig: SDDIM_512_xs_dino}, \ref{fig: SDDIM_512_xxl_dino}, and \ref{fig: SDDIM_512_xxl_fid}, while Examples of generated with deterministic DDIM are given in figures \ref{fig: DDDIM_512_xs_dino}, \ref{fig: DDDIM_512_xxl_dino}, and \ref{fig: DDDIM_512_xxl_fid}.

\begin{table}[htbp]
  \centering
  \caption{FID scores for different sampling schedules on CIFAR10 $32 \times 32$}
  \label{table:CIFAR10-FID}
  \resizebox{0.55\textwidth}{!}{%
  \begin{tabular}{llccc}
    \toprule
    \multirow{2}{*}{Solver} & \multirow{2}{*}{Schedule} 
      & \multicolumn{3}{c}{FID $\downarrow$} \\
    \cmidrule(lr){3-5}
      & & {NFE=16} & {NFE=32} & {NFE=64} \\
    \midrule
    \multirow{3}{*}{Stochastic DDIM} & EDM 
         & \text{33.30} & \text{13.76} & \text{6.36} \\ 
         \cmidrule(lr){2-5}
         & Rescaled Entropy 
         & \text{20.07} & \text{8.44} & \text{4.65} \\
         \cmidrule(lr){2-5}
         & Spectral Rescaled Entropy 
         & \textbf{19.77} & \textbf{8.28} & \textbf{4.47} \\
    \midrule
    \multirow{3}{*}{Deterministic DDIM} & EDM 
         & \text{9.06} & \text{4.18} & \text{2.77} \\
         \cmidrule(lr){2-5}
         & Rescaled Entropy 
         & \text{6.07} & \text{3.30} & \text{2.52} \\
         \cmidrule(lr){2-5}
         & Spectral Rescaled Entropy 
         & \textbf{5.95} & \textbf{3.24} & \textbf{2.51} \\
    \bottomrule
  \end{tabular}%
  }
\end{table}

\begin{table}[htbp]
  \centering
  \caption{FID and FD-DINOv2 scores for different sampling schedules for ImageNet-$64$}
  \label{table: ImageNet-64 FID}
  \resizebox{\textwidth}{!}{%
  \begin{tabular}{lll
                  S[table-format=2.2] S[table-format=2.2] S[table-format=2.2]
                  S[table-format=3.2] S[table-format=3.2] S[table-format=3.2]}
    \toprule
    \multirow{2}{*}{Solver} & \multirow{2}{*}{Network} & \multirow{2}{*}{Schedule} 
      & \multicolumn{3}{c}{FID $\downarrow$} 
      & \multicolumn{3}{c}{FD-DINOv2 $\downarrow$} \\
    \cmidrule(lr){4-6} \cmidrule(lr){7-9}
      & & & {NFE=16} & {NFE=32} & {NFE=64} & {NFE=16} & {NFE=32} & {NFE=64} \\
        \midrule
    \multirow{6}{*}{Stochastic DDIM} & \multirow{3}{*}{EDM2-S} & EDM 
         & \text{20.03} & \text{8.18} & \text{3.81} & \text{263.60} & \text{135.67} & \text{86.32} \\ \cmidrule(lr){3-9}
         & & Rescaled Entropy 
         & \text{11.69} & \text{4.95} & \text{2.75} & \text{194.02} & \text{109.55} & \text{81.25} \\
         \cmidrule(lr){3-9}
         & & Spectral Rescaled Entropy 
         & \textbf{11.46} & \textbf{4.76} & \textbf{2.70} & \textbf{193.81} & \textbf{109.16} & \textbf{79.75} \\
    \cmidrule(lr){2-9}
    & \multirow{3}{*}{EDM2-L} & EDM 
         & \text{22.60} & \text{9.46} & \text{4.44} & \text{284.74} & \text{141.70} & \text{79.86} \\ \cmidrule(lr){3-9}
         & & Rescaled Entropy 
         & \text{13.56} & \text{5.59} & \text{3.06} & \text{208.27} & \text{108.31} & \text{72.06} \\
         \cmidrule(lr){3-9}
         & & Spectral Rescaled Entropy 
         & \textbf{13.46} & \textbf{5.51} & \textbf{2.99} & \textbf{207.76} & \textbf{106.42} & \textbf{70.37} \\
    \midrule
    \multirow{6}{*}{Deterministic DDIM} & \multirow{3}{*}{EDM2-S} & EDM 
         & \text{5.00} & \text{2.49} & \text{1.90} & \text{128.25} & \text{99.64}  & \textbf{92.88} \\ \cmidrule(lr){3-9}
         & & Rescaled Entropy 
         & \textbf{3.46} & \text{2.15} & \textbf{1.77} & \textbf{117.26} & \textbf{98.28}  & \text{93.34} \\
         \cmidrule(lr){3-9}
         & & Spectral Rescaled Entropy 
         & \text{3.54} & \textbf{2.12} & \text{1.80} & \text{118.50} & \text{99.01} & \text{95.14} \\
    \cmidrule(lr){2-9}
    & \multirow{3}{*}{EDM2-L} & EDM 
         & \text{5.49} & \text{2.55} & \text{1.82} & \text{120.35} & \text{84.57}  & \textbf{74.87} \\ \cmidrule(lr){3-9}
         & & Rescaled Entropy 
         & \textbf{3.63} & \textbf{2.09} & \text{1.65} & \textbf{104.88} & \textbf{81.98}  & \text{75.87} \\
         \cmidrule(lr){3-9}
         & & Spectral Rescaled Entropy 
         & \text{3.75} & \textbf{2.09} & \textbf{1.61} & \text{106.52} & \text{82.64} & \text{75.16} \\
    \bottomrule
  \end{tabular}%
  }
\end{table}

\begin{table}[htbp]
  \centering
  \caption{FID and FD-DINOv2 scores for different sampling schedules for ImageNet-$512$ for models optimized for FID scores}
  \label{table: 512 FID}
  \resizebox{\textwidth}{!}{%
  \begin{tabular}{lll
                  S[table-format=2.2] S[table-format=2.2] S[table-format=2.2]
                  S[table-format=3.2] S[table-format=3.2] S[table-format=3.2]}
    \toprule
    \multirow{2}{*}{Solver} & \multirow{2}{*}{Network} & \multirow{2}{*}{Schedule} 
      & \multicolumn{3}{c}{FID $\downarrow$} 
      & \multicolumn{3}{c}{FD-DINOv2 $\downarrow$} \\
    \cmidrule(lr){4-6} \cmidrule(lr){7-9}
      & & & {NFE=16} & {NFE=32} & {NFE=64} & {NFE=16} & {NFE=32} & {NFE=64} \\
        \midrule
    \multirow{4}{*}{Stochastic DDIM} & \multirow{2}{*}{EDM2-XS} & EDM 
         & \text{32.31} & \text{10.01} & \text{4.98} & \text{419.27} & \text{199.69} & \text{131.94} \\ \cmidrule(lr){3-9}
         & & Rescaled Entropy 
         & \textbf{13.64} & \textbf{4.98} & \textbf{3.80} & \textbf{280.82} & \textbf{154.24} & \textbf{124.07} \\
    \cmidrule(lr){2-9}
    & \multirow{2}{*}{EDM2-XXL} & EDM 
         & \text{30.39} & \text{8.80} & \text{3.81} & \text{337.23} & \text{127.69} & \text{68.97} \\ \cmidrule(lr){3-9}
         & & Rescaled Entropy 
         & \textbf{13.38} & \textbf{3.83} & \textbf{2.60} & \textbf{186.68} & \textbf{82.47} & \textbf{57.41} \\
    \midrule
    \multirow{4}{*}{Deterministic DDIM} & \multirow{2}{*}{EDM2-XS} & EDM 
         & \text{10.42} & \text{4.81} & \text{3.83} & \textbf{212.78} & \textbf{154.41} & \textbf{137.68} \\ \cmidrule(lr){3-9}
         & & Rescaled Entropy 
         & \textbf{7.57} & \textbf{4.44} & \textbf{3.75} & \text{222.15} & \text{161.79}  & \text{142.22} \\
    \cmidrule(lr){2-9}
    & \multirow{2}{*}{EDM2-XXL} & EDM 
         & \text{9.68} & \text{3.47} & \text{2.41} & \text{137.35} & \text{81.26}  & \textbf{65.75} \\ \cmidrule(lr){3-9}
         & & Rescaled Entropy 
         & \textbf{5.91} & \textbf{2.78} & \textbf{2.14} & \textbf{125.35} & \textbf{79.80}  & \text{66.75} \\
    \bottomrule
  \end{tabular}%
  }
\end{table}

\begin{table}[htbp]
  \centering
  \caption{FID and FD-DINOv2 scores for different sampling schedules for ImageNet-$512$ for models optimized for DINO scores}
  \label{table: 512 DINO}
  \resizebox{\textwidth}{!}{%
  \begin{tabular}{lll
                  S[table-format=2.2] S[table-format=2.2] S[table-format=2.2]
                  S[table-format=3.2] S[table-format=3.2] S[table-format=3.2]}
    \toprule
    \multirow{2}{*}{Solver} & \multirow{2}{*}{Network} & \multirow{2}{*}{Schedule} 
      & \multicolumn{3}{c}{FID $\downarrow$} 
      & \multicolumn{3}{c}{FD-DINOv2 $\downarrow$} \\
    \cmidrule(lr){4-6} \cmidrule(lr){7-9}
      & & & {NFE=16} & {NFE=32} & {NFE=64} & {NFE=16} & {NFE=32} & {NFE=64} \\
        \midrule
    \multirow{4}{*}{Stochastic DDIM} & \multirow{2}{*}{EDM2-XS} & EDM 
         & \text{17.14} & \text{8.76} & \text{7.34} & \text{294.25} & \text{149.91} & \text{107.00} \\ \cmidrule(lr){3-9}
         & & Rescaled Entropy 
         & \textbf{7.76} & \textbf{6.39} & \textbf{6.43} & \textbf{182.11} & \textbf{109.68} & \textbf{97.10} \\
    \cmidrule(lr){2-9}
    & \multirow{2}{*}{EDM2-XXL} & EDM 
         & \text{15.94} & \text{7.36} & \text{6.06} & \text{218.10} & \text{95.21} & \text{60.79} \\ \cmidrule(lr){3-9}
         & & Rescaled Entropy 
         & \textbf{6.84} & \textbf{5.02} & \textbf{5.03} & \textbf{108.16} & \textbf{57.05} & \textbf{46.75} \\
    \midrule
    \multirow{4}{*}{Deterministic DDIM} & \multirow{2}{*}{EDM2-XS} & EDM 
         & \text{7.16} & \text{5.97} & \text{5.80} & \textbf{156.46} & \textbf{115.94}  & \text{107.05} \\ \cmidrule(lr){3-9}
         & & Rescaled Entropy 
         & \textbf{6.00} & \textbf{5.31} & \textbf{5.40} & \text{157.32} & \text{116.52}  & \textbf{106.84} \\
    \cmidrule(lr){2-9}
    & \multirow{2}{*}{EDM2-XXL} & EDM 
         & \text{5.48} & \text{4.10} & \text{4.03} & \text{79.56} & \text{52.60}  & \text{46.27} \\ \cmidrule(lr){3-9}
         & & Rescaled Entropy 
         & \textbf{3.81} & \textbf{3.54} & \textbf{3.70} & \textbf{68.36} & \textbf{48.26}  & \textbf{43.83} \\
    \bottomrule
  \end{tabular}%
  }
\end{table}

\begin{figure}[H]
    \centering
    \includegraphics[width=0.5\linewidth]{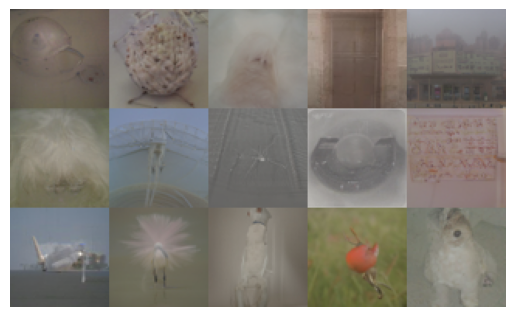}
    \caption{Images generated with the deterministic DDIM solver using the non-rescaled entropic schedule over 64 steps, with the EDM2-L model. It is clear from these images that rescaling is crucial in the continuous regime, probably due to the divergence of the differential entropy at $t \rightarrow 0$.}
    \label{fig: Entropy}
\end{figure}

\begin{figure}[htbp]
    \centering
    \begin{subfigure}[b]{0.47\textwidth}
        \centering
    \includegraphics[width=\textwidth]{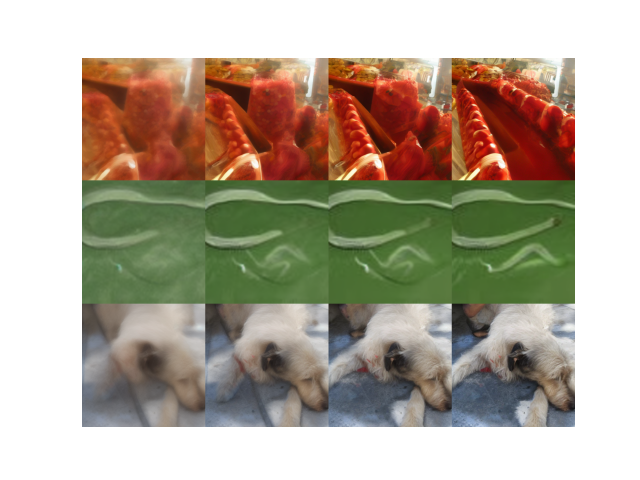}
        \caption{EDM}
    \end{subfigure}
    \hfill 
    \begin{subfigure}[b]{0.47\textwidth}
        \centering
    \includegraphics[width=\textwidth]{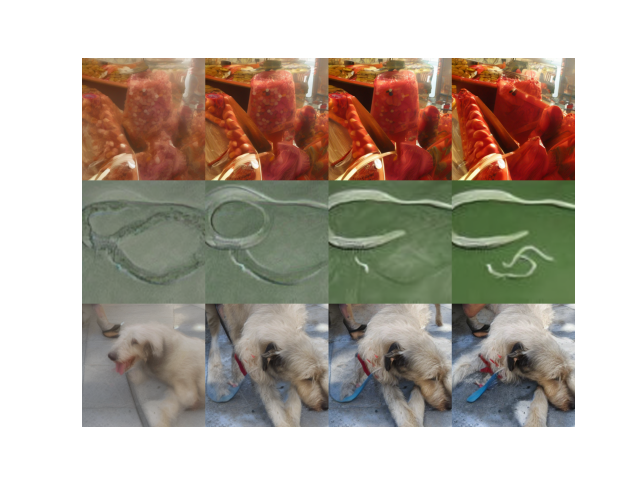}
        \caption{Rescaled Entropy}
    \end{subfigure}
    \caption{Comparison of generated images using EDM and rescaled entropic schedules with the same random seed. Images were generated using deterministic DDIM with NFE = 8, 16, 32, and 64.}
    \label{fig: example of generated images 2}
\end{figure}

\begin{figure}[htbp]
    \centering
    \begin{subfigure}[b]{0.45\textwidth}
        \centering
        \includegraphics[width=\textwidth]{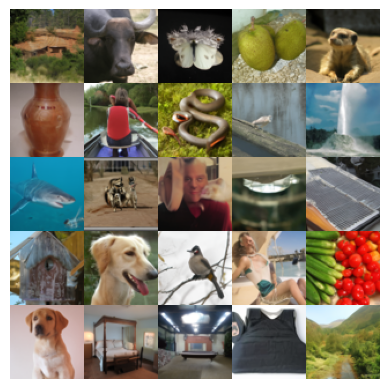}
        \caption{EDM}
    \end{subfigure}
    \hfill 
    \begin{subfigure}[b]{0.45\textwidth}
        \centering
        \includegraphics[width=\textwidth]{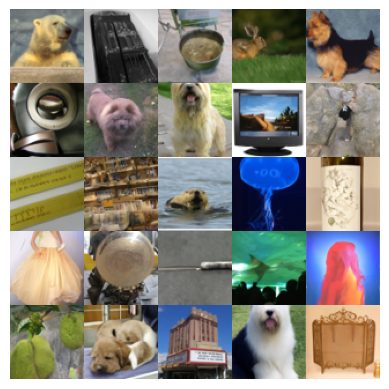}
        \caption{Rescaled Entropy}
    \end{subfigure}
    \caption{Images generated with the stochastic DDIM solver using the EDM schedule (left) and rescaled entropic schedule (right) over 64 steps, with the EDM2-S model.}
    \label{fig: 64_s}
\end{figure}

\begin{figure}[htbp]
    \centering
    \begin{subfigure}[b]{0.45\textwidth}
        \centering
        \includegraphics[width=\textwidth]{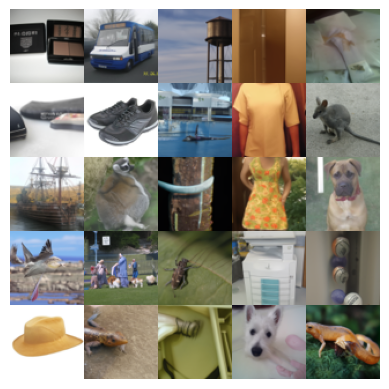}
        \caption{EDM}
    \end{subfigure}
    \hfill 
    \begin{subfigure}[b]{0.45\textwidth}
        \centering
        \includegraphics[width=\textwidth]{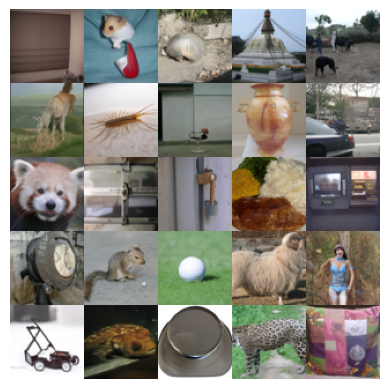}
        \caption{Rescaled Entropy}
    \end{subfigure}
    \caption{Images generated with the stochastic DDIM solver using the EDM schedule (left) and rescaled entropic schedule (right) over 64 steps, with the EDM2-L model.}
    \label{fig: 64_l}
\end{figure}

\begin{figure}[htbp]
    \centering
    \begin{subfigure}[b]{0.45\textwidth}
        \centering
        \includegraphics[width=\textwidth]{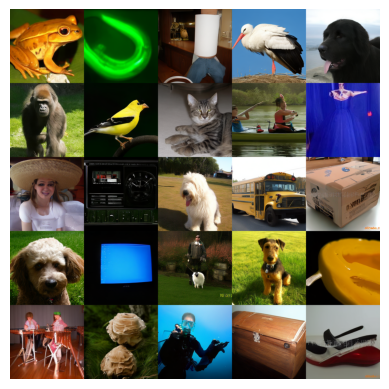}
        \caption{EDM}
    \end{subfigure}
    \hfill 
    \begin{subfigure}[b]{0.45\textwidth}
        \centering
        \includegraphics[width=\textwidth]{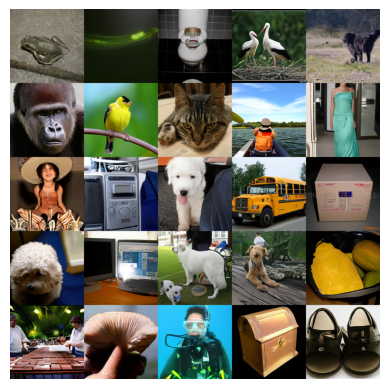}
        \caption{Rescaled Entropy}
    \end{subfigure}
    \caption{Images generated with the stochastic DDIM solver using the EDM schedule (left) and rescaled entropic schedule (right) over 64 steps, with the EDM2-XS DINO-optimized model.}
    \label{fig: SDDIM_512_xs_dino}
\end{figure}

\begin{figure}[htbp]
    \centering
    \begin{subfigure}[b]{0.45\textwidth}
        \centering
        \includegraphics[width=\textwidth]{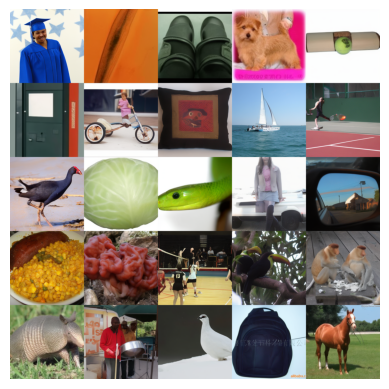}
        \caption{EDM}
    \end{subfigure}
    \hfill 
    \begin{subfigure}[b]{0.45\textwidth}
        \centering
        \includegraphics[width=\textwidth]{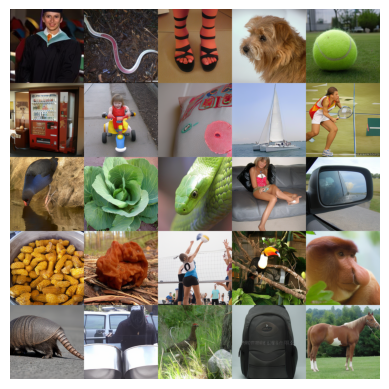}
        \caption{Rescaled Entropy}
    \end{subfigure}
    \caption{Images generated with the stochastic DDIM solver using the EDM schedule (left) and rescaled entropic schedule (right) over 64 steps, with the EDM2-XXL DINO-optimized model.}
    \label{fig: SDDIM_512_xxl_dino}
\end{figure}

\begin{figure}[htbp]
    \centering
    \begin{subfigure}[b]{0.45\textwidth}
        \centering
        \includegraphics[width=\textwidth]{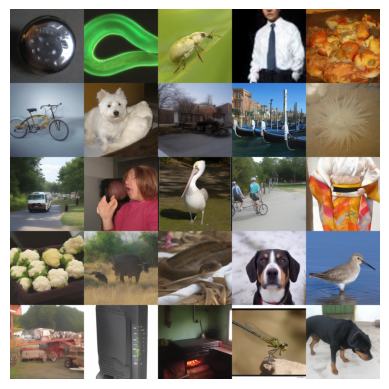}
        \caption{EDM}
    \end{subfigure}
    \hfill 
    \begin{subfigure}[b]{0.45\textwidth}
        \centering
        \includegraphics[width=\textwidth]{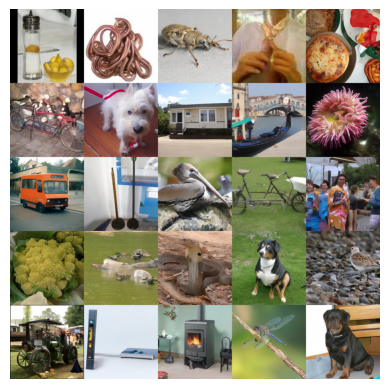}
        \caption{Rescaled Entropy}
    \end{subfigure}
    \caption{Images generated with the stochastic DDIM solver using the EDM schedule (left) and rescaled entropic schedule (right) over 64 steps, with the EDM2-XXL FID-optimized model.}
    \label{fig: SDDIM_512_xxl_fid}
\end{figure}

\begin{figure}[htbp]
    \centering
    \begin{subfigure}[b]{0.45\textwidth}
        \centering
        \includegraphics[width=\textwidth]{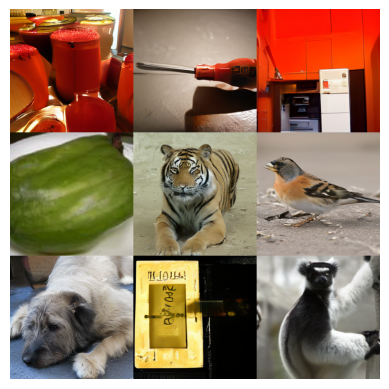}
        \caption{EDM}
    \end{subfigure}
    \hfill 
    \begin{subfigure}[b]{0.45\textwidth}
        \centering
        \includegraphics[width=\textwidth]{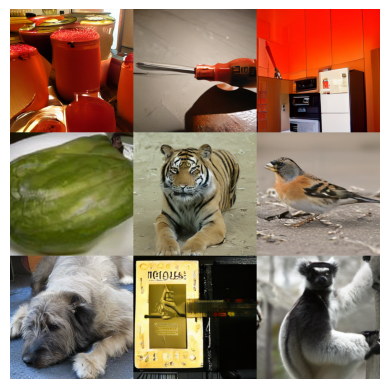}
        \caption{Rescaled Entropy}
    \end{subfigure}
    \caption{Images generated with the deterministic DDIM solver using the EDM schedule (left) and rescaled entropic schedule (right) over 64 steps, with the EDM2-XS DINO-optimized model.}
    \label{fig: DDDIM_512_xs_dino}
\end{figure}

\begin{figure}[htbp]
    \centering
    \begin{subfigure}[b]{0.45\textwidth}
        \centering
        \includegraphics[width=\textwidth]{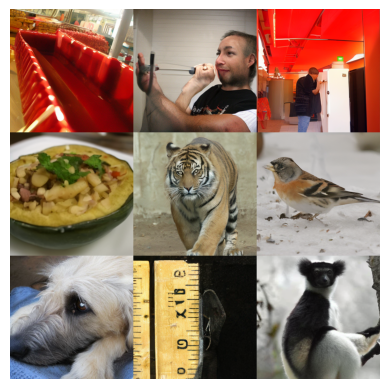}
        \caption{EDM}
    \end{subfigure}
    \hfill 
    \begin{subfigure}[b]{0.45\textwidth}
        \centering
        \includegraphics[width=\textwidth]{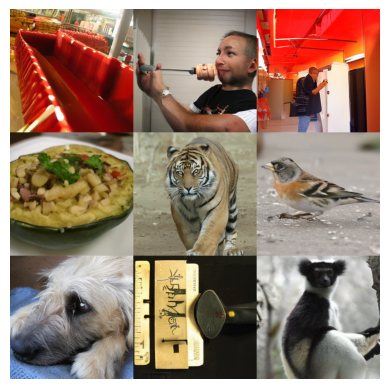}
        \caption{Rescaled Entropy}
    \end{subfigure}
    \caption{Images generated with the deterministic DDIM solver using the EDM schedule (left) and rescaled entropic schedule (right) over 64 steps, with the EDM2-XXL DINO-optimized model.}
    \label{fig: DDDIM_512_xxl_dino}
\end{figure}

\begin{figure}[htbp]
    \centering
    \begin{subfigure}[b]{0.45\textwidth}
        \centering
        \includegraphics[width=\textwidth]{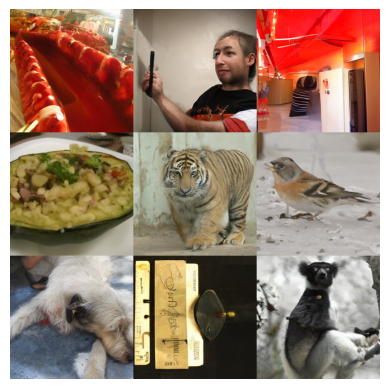}
        \caption{EDM}
    \end{subfigure}
    \hfill 
    \begin{subfigure}[b]{0.45\textwidth}
        \centering
        \includegraphics[width=\textwidth]{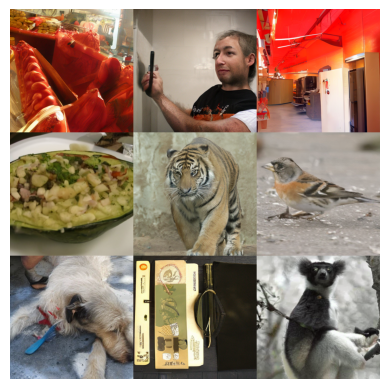}
        \caption{Rescaled Entropy}
    \end{subfigure}
    \caption{Images generated with the deterministic DDIM solver using the EDM schedule (left) and rescaled entropic schedule (right) over 64 steps, with the EDM2-XXL FID-optimized model.}
    \label{fig: DDDIM_512_xxl_fid}
\end{figure}

\end{document}